\documentclass[lettersize,journal]{IEEEtran}
\usepackage{amsmath,amsfonts,amsthm}
\usepackage{algorithmic}
\usepackage{algorithm}
\usepackage{array}
\usepackage[caption=false,font=normalsize,labelfont=sf,textfont=sf]{subfig}
\usepackage{textcomp}
\usepackage{stfloats}
\usepackage{url}
\usepackage{verbatim}
\usepackage{graphicx}
\usepackage{cite}
\usepackage{multirow}
\newtheorem{theorem}{Theorem}
\newtheorem{definition}{Definition}
\newtheorem{lemma}{Lemma}

\newtheorem{proposition}{Proposition}

\DeclareMathOperator*{\E}{\mathbb{E}}

\DeclareMathOperator*{\argmax}{arg\,max}
\DeclareMathOperator*{\argmin}{arg\,min}

\newcommand\inner[2]{\langle #1, #2 \rangle}

\newcommand{\RomanNumeralCaps}[1]
    {\MakeUppercase{\romannumeral #1}}

\newcommand{\mycomment}[1]{}

\begin{document}

\title{Compositional Semantics for Open Vocabulary Spatio-semantic Representations}

\author{Robin Karlsson$^1$, Francisco Lepe-Salazar$^2$, Kazuya Takeda$^{1,3}$
\thanks{Manuscript received October 4, 2023}
\thanks{Revised manuscript received February 25, 2024}
\thanks{This work was supported by JST SPRING,
Grant Number JPMJSP212.}%
\thanks{$^{1}$Robin Karlsson and Kazuya Takeda are with the Graduate School of Informatics, Nagoya University, Japan (karlsson.robin@g.sp.m.is.nagoya-u.ac.jp)}%
\thanks{$^{2}$Francisco Lepe-Salazar is with Ludolab, M\'exico (flepe@ludolab.org).}%
\thanks{$^{3}$Kazuya Takeda is also with TIER IV, Japan (kazuya.takeda@tier4.jp).}
\thanks{Code will be publicly available upon acceptance}%
}



\maketitle

\begin{abstract}
Vision-language models (VLMs) transform environment percepts into vision-language semantics interpretable by LLMs.
However, completing complex tasks often requires reasoning about information beyond what is currently perceived.
We propose latent compositional semantic embeddings $z^*$ as a principled learning-based knowledge representation for queryable spatio-semantic memories.
We mathematically prove that $z^*$ can always be found, and that the optimal $z^*$ is the centroid for any set $\mathcal{Z}$. We derive a probabilistic bound for estimating separability of related and unrelated semantics. We prove that $z^*$ is discoverable  from visual appearance and singular descriptions by iterative gradient descent.
We experimentally verify our findings on four embedding spaces including CLIP and SBERT. Our results show that $z^*$ can represent up to 10 semantics encoded by SBERT, and up to 100 semantics for ideal uniformly distributed high-dimensional embeddings.
We introduce three new datasets with overlapping semantics to show that common VLMs trained on conventional nonoverlapping annotations discover $z^*$. Our novel sufficient similarity inference method overcomes fundamental limitations of conventional inference, and improves higher-level overlapping semantic inference performance by 19.63 mIoU on average. 

\end{abstract}


\section{Introduction}
\label{sec:introduction}

General-purpose mobile robots promise machines capable of safely completing tasks in novel environments without relying on exact human programmed instructions. A promising approach to realize general-purpose robots is to leverage large language models (LLMs)~\cite{ahn2022saycan, shah2022lm_nav, huang2022lm_zero_shot_planners, zeng2023socraticmodels, huang2023inner_monologue, liang2022code_as_policies, nottingham2023deckard, singh2023progprompt, brohan2023rt2} trained on internet scale information about the world.
LLMs compliment the weaknesses of conventional human programmed robots by enabling weakly specified goal definitions in natural language~\cite{ahn2022saycan}, hierarchical planning by program synthesis~\cite{liang2022code_as_policies, nottingham2023deckard, singh2023progprompt}, and reasoning with commonsense knowledge~\cite{brohan2023rt2}.

\begin{figure}
\centering
\includegraphics[width=0.48\textwidth] {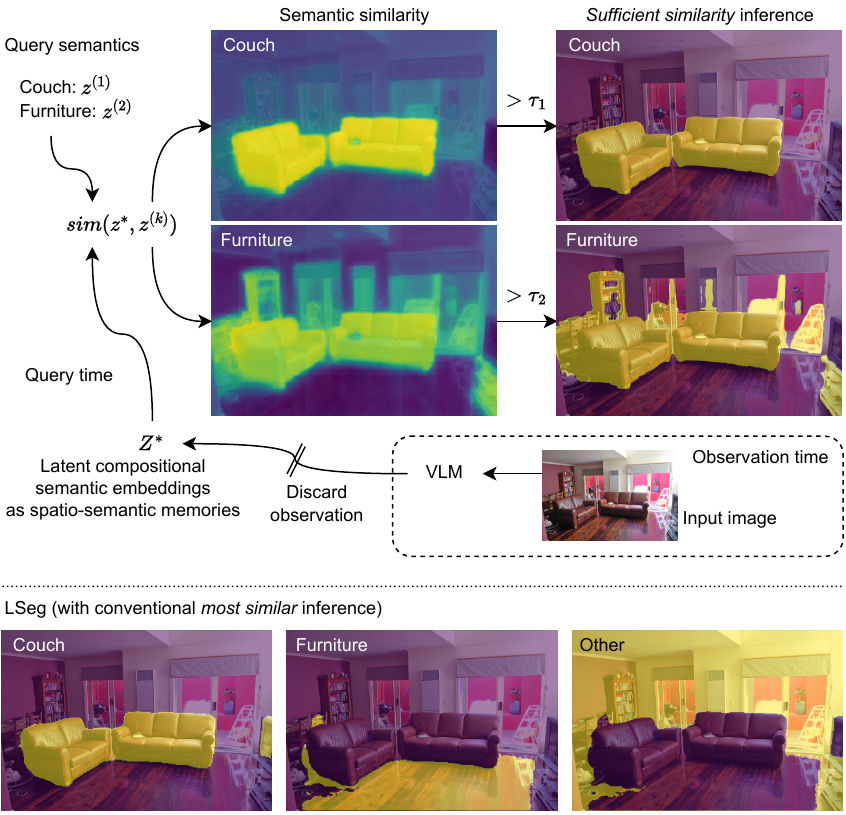}
\caption{We show that unconditional open vocabulary semantic segmentation VLM models learn to map images into latent compositional semantic embedding maps $Z^*$.
Our sufficient similarity inference method allows predicting overlapping semantics for any set of queried semantics $\{z^{(k)}\}$ by similarity with $z^*$, without requiring original input images. 
Conventional unconditional models like LSeg~\cite{li2022lseg} fail at inferring semantic overlap (\textit{couch} is also \textit{furniture}) and incomplete partitionings (\textit{other} is a flawed substitute for unspecified semantics).
Projecting $Z^*$ to spatial coordinates result in accurate and rich open-vocabulary spatio-semantic memories.}
\label{fig:font_fig}
\vspace{-5mm}
\end{figure}

The world knowledge and planning capabilities of LLMs is grounded in the external environment by aligned multi-modal vision-language models (VLMs)~\cite{radford2021clip, schuhmann2022openclip, li2022lseg, ghiasi2022openseg, xu2021zsseg, rao2021denseclip, zhou2022maskclip, ding2023maskclip, xu2023san, zou2023x_decoder, liang2023ovseg}. VLMs transform sensor percepts into an open set of semantic vision-language (VL) embeddings that are interpretable by similarity with other natural language embeddings. LLMs are thus able to receive and query information about the environment by words and sentences using VLMs as an interface. 
The predicted semantic embeddings need to be spatially grounded into spatio-semantic representations for precise spatial comprehension and reasoning~\cite{shah2022lm_nav, huang2023vlmaps, chen2023nlmapsaycan}. Spatial grounding of VL embeddings in 3D can be done by projecting 2D dense VL embedding maps to point clouds~\cite{ha2022semabs, peng2023OpenScene, jatavallabhula2023conceptfusion} or neural radiance fields (NeRF)~\cite{nur2023clipfields}.

Completing tasks generally requires information beyond what is currently observed. A spatio-semantic cognitive memory~\cite{mcnamara1989cognitive_maps}, or semantic scene representations~\cite{davison2018futuremapping}, enables a mobile robot to query semantic information about prior observations~\cite{ha2022semabs, peng2023OpenScene, jatavallabhula2023conceptfusion}, to navigate~\cite{huang2023vlmaps}, and do planning by language-based reasoning~\cite{brohan2023rt2}. Common spatio-semantic representations for mobile robots are 3D reconstruction~\cite{xia2018gibson_env}, object-centric, topological maps~\cite{chen2021topological_planning}, scene graphs~\cite{armeni20193d_scene_graph}, and top-down metric grid maps~\cite{karlsson2023pred_wm}.

The spatio-semantic environment representation for general-purpose mobile robots needs to satisfy three properties:
First, the representation needs to encode rich open-set semantic object descriptions~\cite{wu2009perception_and_concepts, barsalou2012human_conceptual_system, eysenck2020cognitive}. For narrow problems like object avoidance in constrained environments, it may suffice to detect and represent an object by one of a fixed set of classes like \textit{table}. A general-purpose agent~\cite{brohan2023rt2} however, requires a richer compositional representation of the object including alternative names like \textit{desk}, properties like \textit{rigid}, and affordances like \textit{flat surface}, all of which cannot be manually annotated during the system development phase.
Secondly, the representation needs to support querying of overlapping semantics, such as a \textit{dog} also being an \textit{animal}. Overlapping semantics must be learnable from independent observations or datasets without relying on human customization effort limiting scaleability~\cite{lambert2020mseg}.
Third, the representation must be efficient in terms of storage. Spatio-temporal accumulations of raw observations rapidly grow into an unreasonable amount of data~\cite{davison2018futuremapping}. To keep the environment representation compact, observations need to be abstracted into declarative semantic memories~\cite{rosch1976objects_in_natural_categories, binder2011neurobiology_sem_mem, eysenck2020cognitive}.
An additional practically beneficial property is explicit environmental representation. Explicit representations can communicate to humans robots' environmental understanding, intended plan of actions, along with interpretable factors for decision making. Explicit representations also allow humans to provide precise spatially grounded instructions to robots.


In this paper, we investigate latent compositional semantics as a means to compactly represent objects by rich semantic descriptions within explicit environment representations. 
We prove that mathematical properties of high-dimensional hyper-spheres enable a single compositional semantic embedding $z^*$ to define a set of semantic text descriptions encoded into semantic embeddings $\mathcal{Z} = \{z^{(1)}, z^{(K)}\}$.
Our experiments verify that a single embedding $z^*$ can robustly represent 10 semantically related real-world embedded text descriptions, and up to 100 randomly sampled embeddings for ideal uniformly distributed embedding spaces.
Based on our findings, we propose a new perspective on unconditioned dense VL embedding prediction models~\cite{li2022lseg} as a scalable, robust, and learnable neural approximations of semantic networks~\cite{Quillian1961semantic_networks} for knowledge representation.

Our contributions are three-fold:
\begin{itemize}
    \item A mathematical analysis proving that latent compositional semantic embeddings $z^*$ is a principled representation for rich object descriptions by a set of semantics $\mathcal{Z}$. The optimal $z^*$ for $\mathcal{Z}$ is simply the centroid of $\mathcal{Z}$ and does not require contrastive learning.
    \item An empirical investigation of four VL embedding spaces in terms of uniformity, alignment, and their capacity to represent compositional semantics by $z^*$.
    \item Experiments proving $z^*$ are discoverable from visual appearance and singular descriptions by training unconditional dense VLMs and queried by sufficient similarity.
\end{itemize}

The rest of the paper is organized as follows. Sec.~\ref{sec:related_works} explains how compositional semantics connects several fields of artificial intelligence. In Sec.~\ref{sec:compsitional_semantics} we introduce compositional semantics and mathematical properties. A brief presentation of dense VLMs used for inferring compositional semantics from observations is given in Sec.~\ref{sec:dense_VLMs}. We explain experiments and results in Sec.~\ref{sec:experiments}-\ref{sec:results}. Finally, we summarize our findings in Sec.~\ref{sec:conclusions}.



\section{Background and Related works}
\label{sec:related_works}

\noindent \textbf{Knowledge representation.}
A general-purpose intelligent agent needs to store information about the world in a practically useful form for reasoning and task completions. This problem is called knowledge representations. An ontology is a framework for organizing and representing knowledge into a hierarchy of categories or concepts.
Philosophers and artificial intelligence scientists commonly recognize six types of knowledge~\cite{russell2020aima4}: concrete objects including things and stuff, abstract categories for organizing objects in terms of similarity by shared properties, measurements for ordering of properties, and events, fluents, and time points specify temporally changeable statements.


First-order logic (FOL)~\cite{McCarthy1958FOL_AI} and extensions like fuzzy~\cite{Zadeh1965fuzzy_logic} and modal logic~\cite{Kripke1959modal_logic} traditionally express an object $x$ being a member of a category $Category$ as $Category(x)$ or $x \in Category$. 
Semantic networks~\cite{Quillian1961semantic_networks} is a subset of FOL designed to represent knowledge as a directed graph of objects and categories. Objects are associated to one or more categories by $MemberOf(\cdot, \cdot)$ relations. Categories are associated to other categories to form a taxonomic hierarchy. The hierarchy of categories allows objects to inherit semantic descriptions from higher-level categories, implying that an object that is a chicken is also a bird (but not the other way around):
\begin{equation}
    Chicken(x) \Rightarrow Bird(x).
\end{equation}
Frames~\cite{Minsky1975frames} extend Semantic networks with inheritable default attribute values like $height=1$ and properties $CanFly=True$ similar to object-oriented programming.

Semantic networks have several practical limitations.
First, semantic vagueness is an inherent aspect of object descriptions as explained in Sec.~\ref{sec:introduction}. Expressing degree of membership is challenging in purely logical representations.
Secondly, the problem of inferring correct and diverse category associations from perception is not addressed.
Finally, a complete ontology encompassing the entire world does not exist. A scalable method for learning and revising a diverse set of category associations from incomplete and noisy data is needed.

We propose compositional semantic embeddings as a principled and scalable approach to learn compact and semantically diverse object descriptions from uncurated data. 

\noindent \textbf{Natural language processing.}
The study of using natural language as an interface for human-machine communication, and how to enable machines to leverage human written knowledge, is called natural language processing (NLP). Natural language is ambiguous and sentence correctness is not perfectly decidable by rules~\cite{russell2020aima4}. Language models (LM)~\cite{devlin2019bert, reimers2019sbert} instead learn to predict the likelihood $p(\mathcal{X})$ of any sequence of text tokens~$\mathcal{X}$ according to a natural language dataset.

Word embeddings~\cite{mikolov2013wordvec, pennington2014glove} substitute non-semantic word tokens by a semantic vector representing the meaning of the word. Word embeddings are discovered from maximizing similarity of embeddings of co-occuring words~\cite{harris1954wc_co_occurance}.
Contextual representations~\cite{peters2018contextual_word_rep} extends word embeddings by encoding context from surrounding words.
Large langauge models (LLM)~\cite{quoc2014distr_repr_sentences, devlin2019bert, reimers2019sbert} can generate semantic embeddings out of entire sentences.
Our approach differs from word and sentence embeddings as we represent a set of semantic embeddings representing an object description by a single compositional semantic embedding.

Clustering~\cite{maas2010clustering_word_vec} and mixture models~\cite{blei2003lda, hoffman2010online_lda} in NLP discover groups of semantically similar text data. The Latent Dirichlet Allocation (LDA) model~\cite{blei2003lda} parses documents into mixtures of discovered latent topics that allow a finer semantic similarity search. 
A generative probabilistic mixture model $p(z)$ approximates the distributions of semantic embeddings $z \in \mathcal{Z}$ by $K$ mixture components $p_k(z)$ weighted by the probability $\pi_k$ that each mixture component is sampled
\begin{equation}
    p(z) = \sum_{k=1}^K \pi_k \; p_k(z).
\end{equation}
The optimal model is a mixture of Dirac delta distributions $p_k(z) = \delta(z - z^{(k)})$ with number of mixtures $K$ equaling the number of semantics in the distribution $\mathcal{Z}$. As the set of possible semantics in natural languages are unbounded, a common distribution approximation is the Gaussian mixture model (GMM) with $K \ll |\mathcal{Z}|$ components $p_k(z) = \mathcal{N}(\mu_k, \Sigma_k)$ representing the $K$ best semantic clusters.
However, this approximation has practical limitations. The required clusters $K$ is generally not known. Optimizing the mixture model $p(z)$ is challenging. Storing the the mixture distribution parameters or all $K$ semantic embeddings $\mu_k$ can be inefficient.

Our compositional semantics approach instead leverage properties of high-dimensional hyperspheres to find an optimal semantic embedding $z^*$ akin to clustering. The vector $z^*$ defines $p(z \in \mathcal{Z})$ by similarity instead of approximating the entire distribution $p(z)$. Our approach has mathematical guarantees of optimality, and can represent a large set of semantics by a single embedding while optimizable by gradient descent.

\noindent \textbf{Vision-language modeling.}
Multimodal models that semantically interpret images and text by a unified embedding space are called vision-language models (VLMs). 
Global description generating VLMs~\cite{radford2021clip} consist of a visual $Enc_V()$ and language encoder $Enc_L()$. Both encoders are co-trained to generate a semantically similar visual and text embedding $z_v$ and $z_t$ for an input image $x$ and text $t$ in an aligned embedding space $Z$. Alignment enables VLMs to be used as an interface to query or express contents of visual data in natural language. Semantic correspondence between $z_v$ and $z_l$ is measured by cosine similarity. The encoders are typically trained on internet-scale image captioning datasets using contrastive learning.
Global description models have many usages like image-text matching, multimodal search, multimodal generative modeling~\cite{li2023blip2}, and visual-question-answering~\cite{liu2023llava}.
However, outputs are not spatially grounded in the input image and therefore have limitations for tasks requiring precise spatial information such as navigation~\cite{shah2022lm_nav}, manipulation~\cite{ahn2022saycan}, and mapping~\cite{huang2023vlmaps}.

Dense description VLMs~\cite{li2022lseg, ghiasi2022openseg, xu2021zsseg, rao2021denseclip, zhou2022maskclip, ding2023maskclip, xu2023san, zou2023x_decoder, liang2023ovseg} generate aligned embeddings for every image pixel for fitting semantics to object boundaries. 
MaskCLIP~\cite{zhou2022maskclip} aims to leverage the strong generalization power of global description VLMs by removing the global pooling layer.
However, the output is considerably noisy and have limited practical usefulness for robotics tasks.

One approach to generate dense descriptions is to use a region proposal (RP) model~\cite{cheng2022mask2former}. The RP model predicts a set of object crops that are interpreted by a global VLM~\cite{xu2021zsseg}. The resulting global embedding is projected onto all pixels covered by the region.
The object crop approach works well for object-centered image inputs typical for indoor robotics environments, but less so for large and complex scenes requiring multi-scale object perception~\cite{zhong2022regionclip, liang2023ovseg}. Computational cost is high due to performing inference for every region separately.

Another direction of work instead trains a new vision model $f_\theta()$ with and architecture and optimization scheme designed for dense feature representation.
%
LERF~\cite{kerr2023lerf} grounds language embeddings in a neural radiance field (NeRF)~\cite{mildenhall2020nerf}, allowing querying semantics in 3D environment representations.
%
Open-vocabulary (OV) object detectors localize predicted VL embeddings to bounding boxes~\cite{gu2022vild}.
Works related to open-vocabulary semantic segmentation can be categorized into two types. Conditional OV semantic segmentation models~\cite{rao2021denseclip, ding2022zegformer, lueddecke2022clipseg, zou2023x_decoder} allows fine-grained query guided by additional text and/or image input. One drawback is that conditional inference require the original image. Unconditional methods~\cite{ghiasi2022openseg, li2022lseg, xu2023san} learns to predict general embedding maps such that the likelihood is maximized over the training dataset. Contrary to global embedding models~\cite{radford2021clip}, unconditional semantic segmentation models are trained on relatively small densely annotated datasets. The expressiveness of unconditionally predicted embeddings is not yet deeply understood.
%

In this paper we present an interpretation of unconditional OV semantic segmentation predictions as latent compositional semantic embeddings $z^*$. We show that the representation $z^*$ combines the compactness of unconditional inference, the expressiveness of conditional inference, and the capacity to represent semantic object descriptions of length $K$.

\noindent \textbf{Spatio-semantic representations.}
%
Mobile robots typically perform planning for spatial tasks by localizing its pose within a map~\cite{thrun2006stanley}. ICP~\cite{besl1992icp} or SLAM~\cite{randall1986slam} by modern implementations~\cite{jatavallabhula2020grad_slam, vizzo2023kiss_icp} is the conventional approach to map 3D environments by matching sequential point clouds and accumulating them into a common vector space.
Semantic SLAM not only estimate the geometry but also the semantics of the environment or an object~\cite{mccormac2017semanticfusion}.
The 3D representation can be projected onto a 2D birds-eye-view (BEV) map convenient for navigation tasks~\cite{schulter2018learning_to_look_around_obj}. The image-like 2D map representation is suitable for predictive generative modeling~\cite{karlsson2023pred_wm}.
Until recently, semantic mapping approaches were limited to predefined sets of semantic classes, and thus to narrow tasks.

Open-vocabulary spatial representation methods encode maps by VL embeddings instead of class embeddings. The VL embeddings are typically generated by a pretrained global VLM~\cite{jatavallabhula2023conceptfusion}, open-vocabulary object detector~\cite{chen2023nlmapsaycan}, or a dense VLM~\cite{ha2022semabs, huang2023vlmaps, peng2023OpenScene}. The open-vocabulary approach in principle allows querying any task-relevant semantics stored in the VL embeddings measuring cosine similarity with a text query embedding.
NeRFs implicitly represents 3D objects and environments by a neural network~\cite{mildenhall2020nerf, ost2021nerf_dynamic, martinbrualla2020nerf_wild} and have recently been extended represent open-vocabulary semantics~\cite{nur2023clipfields}.
Integrating LLMs opens up new possibilities for spatio-semantic reasoning based on a top-down perceptual feedback loop~\cite{chen2023nlmapsaycan, pi2023detgpt} similar to the human vision-for-perception system~\cite{gibson1979ecological_visual_perception, milner2008two_visual_systems, zhixian2022modeling_ventral_dorsal}.

Our work presents an interpretation of VL embeddings learned by dense open-vocabulary models as latent compositional semantic embeddings $z^*$. We show that $z^*$ can combine the generality of global description VLMs like CLIP~\cite{radford2021clip} with the accuracy of dense VLMs like LSeg~\cite{li2022lseg}.

\noindent \textbf{Cognitive psychology and philosophy.}
Our idea of latent compositional semantics has strong support in cognitive psychology, neuroimaging, and philosophy.
Semantic memories are derived from an agent's experiences but is characterized by its abstract and conceptual nature, devoid of ties to any specific encounter~\cite{binder2011neurobiology_sem_mem}. Our approach implements the idea of semantic memories into a computational learning framework.
Semantic concepts are organized into hierarchies~\cite{eysenck2020cognitive} and processed in relation to perceivable context~\cite{wu2009perception_and_concepts, barsalou2012human_conceptual_system}. Our work shows how machines can learn hierarchical concepts from independent visual observations.
Philosophers argue that real world objects are generally not perfectly described by a single category, but by fuzzy semantic descriptions determined by degree of semantic membership~\cite{wittgenstein1953phil_inv, lakoff1987women_fire_dangerous, schwartz1977naming_necessity_natural_kinds}. We provide a computational framework to learn such object descriptions from incomplete descriptions.
See the supplementary material for further details.

\section{Compositional semantics}
\label{sec:compsitional_semantics}
In this section, we first present the idea of compositional semantics, and how a single vector $z^* \in \mathbb{R}^D$ implicitly represents a diverse set of semantic object descriptions $\mathcal{Z}$.
In Sec.~\ref{sec:comp_properties_uniform}-\ref{sec:comp_properties_nonuniform} we derive properties of compositional semantic embeddings $z^*$ for uniform and non-uniform embedding distributions based on mathematical analysis of high-dimensional hyperspheres.
Finally, in Sec.~\ref{sec:comp_disc_grad_descent} we analyze practical discoverability of compositional semantics in real world VL embeddings spaces by iterative gradient descent.

\subsection{Compositional object representations}

Knowledge representations aim to describe concrete objects by membership to abstract semantic categories. Semantic networks are a common object description representation encumbered by practical limitations. We propose compositional semantics as an efficient and practical vector space representation for describing objects by a potentially large set of semantic categories by a scaleable learning-based method. Compositionality means that complex expressions, such as sentences or functions, can be determined or understood based on the meanings of their individual parts~\cite{jenssen1997compositionality}. 

Our proposed framework for compositional semantics is shown in Fig.~\ref{fig:compositional_semantics}. The objective is to find a hyperspherical latent compositional semantic embedding $z^* \in S^{D-1}$ for an object which is similar to all semantic embeddings in the set $\mathcal{Z} = \{ z^{(1)}, \ldots, z^{(K)} \}$ that broadly describe the object (see Sec.~\ref{sec:introduction}). Semantic similarity is defined in terms of separation distance in the embedding space $ S^{D-1}$. Distances between embeddings on unit hyperspheres in Euclidean vector spaces are conveniently represented by cosine similarity
\begin{equation}
\label{eq:cos_sim}
    \cos \omega = \frac{\langle z^*, z^{(k)} \rangle}{||z^*|| \; ||z^{(k)}||} = (z^*)^T z^{(k)}.
\end{equation}

\begin{figure}
\centering
\includegraphics[width=0.42\textwidth] {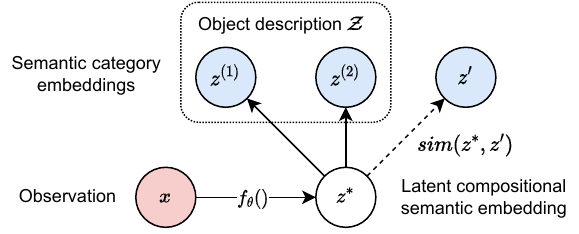}
\vspace{-2mm}
\caption{The compositional semantics framework. An observation $x$ is mapped into an embedding $z^*$ that specifies an object description $\mathcal{Z}$ in terms of interpretable semantic categories $z^{(k)}$ through fuzzy membership by similarity.}
\label{fig:compositional_semantics}
\vspace{-5mm}
\end{figure}

An observation $x$ is mapped into a compositional semantic embedding $z^*$ discovered by a learned one-to-one mapping function $f_\theta(x)$.
The optimal embedding $z^*$ is found by maximizing the mean cosine similarity~\eqref{eq:cos_sim} over all describing semantics $z \in \mathcal{Z}$. We presume the distribution $p(z)$ is approximate uniformly distributed.
In this paper we show that contrastive optimization by minimizing~\eqref{eq:cos_sim} over negative samples $z'$ is not required if $\mathcal{Z}$ is known. The mathematical properties of high-dimensional hyperspheres ensure that any other randomly sampled embeddings is very likely to be dissimilar to $z^*$. The optimal compositional semantic embedding $z^*$ thus separates the set of describing semantics $\mathcal{Z}$ from all other semantics $z' \sim \mathit{U}(S^{D-1})$.
Observations $x$ denote any observable representation including image pixel regions.

During inference, the representation $z^*$ for an observation~$x$, implicitly encodes $\mathcal{Z} = \{z^{(1)}, \ldots, z^{(T)}\}$ concatenated from past independent learning samples $(x^{(t)}, z^{(t)})$. From the perspective of knowledge representation, $z^*$ implicitly encodes the degree of membership for any queried semantic $z$ by semantic distance or equivalently cosine similarity~\eqref{eq:cos_sim} :
\begin{equation}
\label{eq:membership_by_sim}
    MemberOf(x, z) \propto sim(z^*, z) \;\; z^* := f_\theta(x).
\end{equation}
The set of inferred $\hat{\mathcal{M}}$ and original $\mathcal{M}$ set of object descriptions are approximately equal
\begin{gather}
    \hat{\mathcal{M}} = \{ MemberOf(x, \hat{z}) \; | \; \hat{z} \in \hat{\mathcal{Z}} \} \\
    \mathcal{M} = \{MemberOf(x, z) \; | \; z \in \mathcal{Z} \} \\
    |\hat{\mathcal{M}} \cup \mathcal{M}| \simeq |\mathcal{M}|
\end{gather}
as the set of inferred sufficiently similar semantic description embeddings are approximately equal
\begin{equation}
\label{eq:compositional_semantics_membership_set}
    \hat{\mathcal{Z}} = \{ \hat{z} | sim(z^*, \hat{z}) > \tau \;\; \forall \hat{z} \in S^{D-1} \} \simeq \mathcal{Z}.
\end{equation}

The degree of membership by similarity~\eqref{eq:membership_by_sim} reflects the fact that real world objects rarely have a single, clear-cut semantic specification~\cite{wittgenstein1953phil_inv, lakoff1987women_fire_dangerous}.
The threshold of sufficient semantic membership $\tau$ is subjective and needs to be optimized in respect to a purpose or task~\cite{schwartz1977naming_necessity_natural_kinds}.
Note that the mapping $f_\theta(x)$ discovers $z^*$ from independent samples $(x^{(t)}, z^{(t)})$ by iterative gradient descent.

\subsection{Compositional properties for uniform distributions}
\label{sec:comp_properties_uniform}

VL embeddings are typically located on the surface of a high-dimensional unit hypersphere. In this section we analyse the compositional properties of VL embeddings spaces based on mathematics for high-dimensional probability distributions~\cite{vershynin2018high_dim_prob}.
%
%
%
We begin the analysis by formally defining latent compositional semantic embeddings $z^*$.

\begin{definition}\label{def:compositionality} A vector $z^* \in \mathbb{R}^D$ on the unit hypersphere $S^{D-1}$ is a compositional semantic embedding for a set of semantic embeddings $z \in \mathcal{Z}$ if 
\begin{equation}
\label{eq:compositionality}
    \E sim(z^*, z) > \E sim(z^*, z') \;\;\; \forall z \in \mathcal{Z}, z' \sim \mathit{U}(S^{D-1})
\end{equation}
%
\noindent
where $\mathit{U}(S^{D-1})$ is the uniform distribution over $S^{D-1}$.
\end{definition}

The following theorem specify the theoretically optimal $z^*$ embedding is simply a centroid.

\begin{theorem}\label{theorem:discoverability}[Discoverability \RomanNumeralCaps{1}] It is always possible to find the optimal compositional semantic embedding $z^* \in \mathbb{R}^{D\gg1}$ satisfying Definition~\ref{def:compositionality} as the centroid of the set of semantics $\mathcal{Z}$
\begin{equation}
\label{eq:optimal_comp_sem_centroid}
   z^* = \tfrac{1}{K} \sum_{i=1}^K z^{(i)} \;\; \forall z^{(i)} \in \mathcal{Z}. 
\end{equation}
\end{theorem}

\begin{proof}
    See Appendix A.  
\end{proof}

The proof is based on finding the $z^*$ maximizing cosine similarity by partially differentiating the equivalent minimum square distance.

A property of high-dimensional vector spaces is that any two random variable vectors are expected to be approximately orthogonal. The following lemma is used to prove Theorem~\ref{theorem:discoverability}

\begin{lemma}
\label{lemma:expected_sim}[Expected similarity] For two independent random vectors $Z^{(i)}$, $Z^{(j)}$ sampled from an isotropic high-dimensional distribution $Z \in \mathbb{R}^{D}$ with $D \gg 1$
\begin{equation}
\label{eq:expected_sim}
   \E sim(Z^{(i)}, Z^{(j)}) =  \frac{1}{\sqrt{D}}. 
\end{equation}
\end{lemma}

\begin{proof}
    See Appendix B.  
\end{proof}

The proof involves recognizing $Z$ as an isotropic distribution and computing the expectation of a dot product for two random vectors $Z^{(i)}$ and $Z^{(j)}$.

Next we derive a probabilistic bound defining the separability of a set $\mathcal{Z}$ of object descriptions and random descriptions $z'$ by similarity with the latent compositional semantic embedding $z^*$ for $\mathcal{Z}$.

\begin{theorem}
\label{theorem:prob_of_compositionality}[Probabilistic bound] The probability $P$ a compositional semantic embedding $z^*$ is more similar to all its semantic members $z \in \mathcal{Z}$ than any unrelated semantic embedding $z' \sim \mathit{U}(S^{D-1})$ is
\begin{equation}
\label{eq:prob_of_compositionality}
    P\left(sim(z^*, z) > sim(z^*, z')\right) = 1 - \tfrac{1}{2} I_{\sin^2(\theta_{min})}(\tfrac{D-1}{2}, \tfrac{1}{2})
\end{equation}
\noindent
where $I_x(a, b)$ is the regularized incomplete beta function and
\begin{equation}
\label{eq:theta_min}
    \theta_{min} = \arccos( sim(z^*, z_{min}) )
\end{equation}
is the angle $\theta_{min}$ defined by the least similar member
\begin{equation}
    z_{min} = \argmin ( sim(z^*, z) ) \; \forall z \in \mathcal{Z}.    
\end{equation}

\end{theorem}

\begin{proof}
    See Appendix C.  
\end{proof}

 The proof is based on noting that the probability $P$ a randomly sampled unrelated embedding $z'$ falsely in the set of semantic members $\mathcal{Z}$ is proportional to the area ratio of the hyperspherical cap $S^{D-1}_{cap}$ spanned by $z^*$ and $z_{min}$. The proof builds upon Lemma~\ref{lemma:expected_sim}~and~\ref{lemma:hyperspherical_cap}.

\begin{lemma}
\label{lemma:hyperspherical_cap}[Hyperspherical cap] The compositional semantic embedding $z^*$ and all semantic member embeddings $z \in \mathcal{Z}$ lie in a hyperspherical cap $S^{D-1}_{cap}$
\begin{equation}
\label{eq:hyperspherical_cap}
    \{z^*\} \cup \mathcal{Z} \in S^{D-1}_{cap} = \{ z \in \mathbb{R}^D : \|z\| = 1, \theta_z \le \theta_{min} \}.
\end{equation}
%
%
%
%
%
\end{lemma}

\begin{proof}
    See Appendix D.  
\end{proof}

We conclude that latent compositional semantic embeddings $z^*$ can always be found for VL embeddings. The goodness of $z^*$ can be measured by the probabilistic estimate ~\eqref{eq:prob_of_compositionality}

\subsection{Compositional properties for non-uniform distributions}
\label{sec:comp_properties_nonuniform}

The mathematical properties for latent compositional semantic embeddings $z^*$ in Sec.~\ref{sec:comp_properties_uniform} are derived for uniformly distributed embeddings. In this section, we analyze the validity of the results for non-uniform hyperspherical distributions.

\begin{proposition}[Discoverability \RomanNumeralCaps{2}]
\label{proposition:discoverability_2}
    It is always possible to find an optimal compositional semantic embedding $z^* \in \mathbb{R}^D$ for any non-uniform distribution $z \in p(z | z \in \mathbb{R}^{D>1}, \|z\| = 1)$ that is not singular.
\end{proposition}

\begin{proof}
    See Appendix E.  
\end{proof}

The proof is based on showing that Definition~\ref{def:compositionality} holds also when expected similarity is higher than for uniformly distributed embeddings spaces as given by Lemma~\ref{lemma:expected_sim}.

The shape of the non-uniform density $p(z)$ of common VLMs is a product of optimization by contrastive learning with random negative sampling~\cite{radford2021clip}.
Few general properties can be inferred for non-uniform densities. So~et~al.\cite{so2022geodesic} finds that vision and text CLIP embeddings are distributed in separate modality-specific hyperspherical caps. Wang~et~al.~\cite{wang2021understanding_cl} identifies the uniformity-alignment dilemma stating that perfect uniformity and alignment cannot be simultaneously achieved due to semantically similar but randomly sampled false negatives.


We found that using the probabilistic bound~\eqref{eq:prob_of_compositionality} for highly non-uniform VL embedding densities $p(z)$ results in poor estimates. The reason is that unrelated embeddings are far more similar than those for uniform distributions.
%
Instead we propose a statistical sampling-based approach to obtain a probabilistic estimate for \eqref{eq:compositionality} in Definition~\ref{def:compositionality} without requiring to estimate the non-uniform density $p(z)$.
The probability in \eqref{eq:prob_of_compositionality} is estimated by sampling $N$ random semantic embeddings $z \sim p(z)$ and counting the number of samples being within the hyperspherical cone $S^{D-1}_{cap}$ spanned by $z^*$ and $z_{min}$ \eqref{eq:hyperspherical_cap} such that 
\begin{equation}
\label{eq:statistical_satisfiability}
    P\left(sim(z^*, z) > sim(z^*, z')\right) \simeq \frac{1}{N} \sum_{i=1}^N\mathbf{1}_{S^{D-1}_{cap}}(z^{(i)}).
\end{equation}

Our empirical results show that latent compositional semantic embeddings $z^*$ are useful for all tested non-uniform VL embedding distributions. Additionally, the empirical estimate~\eqref{eq:statistical_satisfiability} provides an accurate measure of goodness.

\subsection{Compositional discovery by gradient descent}
\label{sec:comp_disc_grad_descent}

The mathematical properties in Sec.~\ref{sec:comp_properties_uniform}-\ref{sec:comp_properties_nonuniform} are derived while presuming all member semantics $z \in \mathcal{Z}$ are known. In this section we verify the possibility of finding latent compositional semantic embeddings $z^*$ by iterative optimizing $z^*$ one $z$ at a time, instead of averaging the set $\mathcal{Z}$ as in~\eqref{eq:optimal_comp_sem_centroid}.

\begin{proposition}[Discoverability \RomanNumeralCaps{3}]
\label{proposition:gradient_descent}

It is always possible to find an optimal compositional semantic embedding $z^* \in \mathbb{R}^D$ by iterative gradient descent optimization
\begin{equation}
    z^{*(t+1)} = z^{*(t)} - \lambda \; \nabla_{z^*} \left[ \sum_{i=1}^L sim(z^{*(t)}, z^{(i)}) \right]
\end{equation}
over random subsets $\tilde{\mathcal{Z}}^{(t)} \subseteq \mathcal{Z}, \; |\tilde{\mathcal{Z}}^{(t)}| = L$ given a sufficiently small learning rate $\lambda$.
\end{proposition}

\begin{proof}
    See Appendix F.  
\end{proof}

The proof is based showing that the cosine similarity optimization objective
is convex, and noting that all convex problems have global convergence guarantees.

\section{Unconditional Dense Vision-language model}
\label{sec:dense_VLMs}

An unconditional dense VLM is a learned one-to-one function $f_{\theta}()$ that maps images $x \in \mathbb{R}^{3 \times H \times W}$ to dense embedding maps $Z \in \mathbb{R}^{D \times H \times W}$ consisting of aligned VL embeddings~$z_{(i,j)} \in \mathbb{R}^D$ at point $(i,j)$ in the image frame. The output $Z$ represents observations abstracted into declarative semantic memories~\cite{rosch1976objects_in_natural_categories, binder2011neurobiology_sem_mem, eysenck2020cognitive} which maximizes the predictive likelihood over past observations given $x$, without conditioning on an input text~\cite{ghiasi2022openseg} or an image query~\cite{lueddecke2022clipseg}.
Unconditional prediction~\cite{radford2021clip, li2022lseg} is necessary for efficient open vocabulary spatio-semantic memory representations as explained in Sec.~\ref{sec:introduction}.
However, since objects have not one but several semantic descriptions~\cite{wu2009perception_and_concepts, barsalou2012human_conceptual_system, eysenck2020cognitive}, a single semantic embedding $z_{(i,j)}$ must simultaneously encode a multitude of task-relevant semantics.

We investigate the feasibility of discovering compositional semantics by $f_\theta$ as an image encoder-decoder dense prediction deep neural network architecture. To maximize the generality of our findings, $f_\theta$ is implemented by conceptually simple, general, and well-performing SOTA modules as shown in Fig.~\ref{fig:dense_vlm_model}. A vision transformer (ViT) backbone~\cite{dosovitskiy2021vit} extracts visual features from image observations $x$. We use the ViT-Adapter~\cite{chen2022vitadapter} as a dense prediction task adapter to enhance the ViT backbone with vision-specific inductive biases. The adapter outputs a set of multi-scale feature maps $\mathcal{F} = \{F_1, F_2, F_3, F_4\}$. A Feature Pyramid Network (FPN)~\cite{lin2017fpn} integrates $\mathcal{F}$ into a single feature map $F$. A simple decoder head bilinearly upsamples $F$ into the input image resolution and do a final $1 \times 1$ convolution to project features into normalized semantic embedding maps $Z$.

\begin{figure}
\centering
\includegraphics[width=0.48\textwidth] {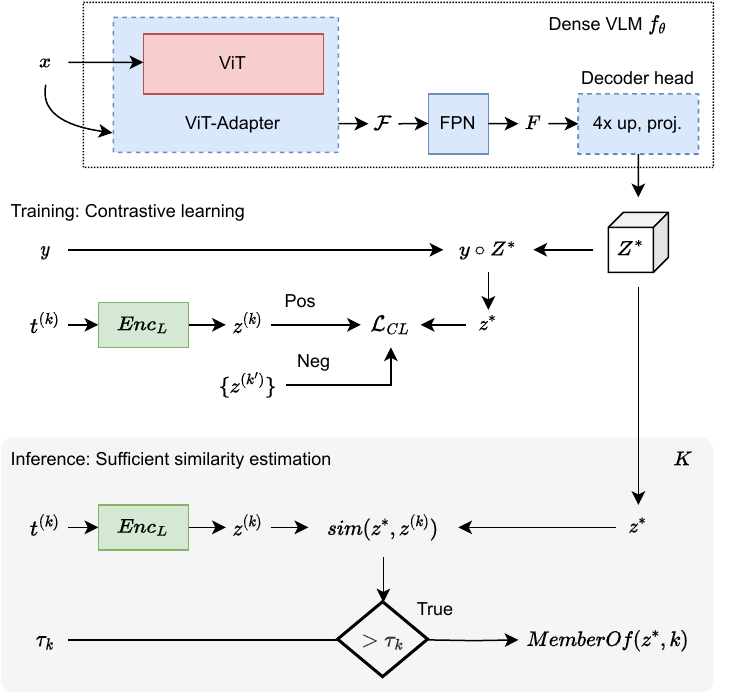}
\vspace{-2mm}
\caption{The unconditional dense VLM $f_\theta$ transforms an image $x$ into an embedding map $Z^*$ representing compositional semantics $z^*$ for every pixel. During training, predictions $z^*$ for elements masked by $y$ are optimized to be similar to targets $z^{(k)}$ and dissimilar to all other semantics $z^{(k')}$ generated from text descriptions $t^{(k)}$ by a language encoder $Enc_L$. During inference, $z^*$ allows querying multiple semantics $K$ by similarity. All elements above the similarity threshold $\tau_k$ are members of the semantic group $k$. $\tau_k$ is set to maximize likelihood of predicting past observations.}
\label{fig:dense_vlm_model}
\vspace{-6mm}
\end{figure}


In the remainder of this section, we explain how in fact dense latent compositional semantic embedding maps $Z^*$ are discovered by an unconditional dense VLM $f_{\theta}$ when trained to predict $Z$.


\subsection{Model training}

The model $f_\theta$ is initialized with pretrained backbone parameters and trained end-to-end to predict semantic embedding maps $Z$ from images $x$ and dense annotations. Annotations consists of $K$ types of paired semantic text descriptions $t^{(k)}$ encoded into semantic embeddings $z^{(k)}$, and boolean image masks $y \in \mathbb{B}^{H \times W}$ specifying which image elements $x_{(i,j)}$ are associated with $t$. We denote an observation $n$ as a tuple $(x, t, y)_n$.

We use the contrastive learning objective
\begin{equation}
\label{eq:cl_objective}
    \mathcal{L}_{CL} = \E \left[ -\log \frac{e^{ sim(z, z^{(k)} ) / \tau}}{ e^{ sim(z, z^{(k)}) / \tau } + \sum_{k'} e^{ sim(z, z^{(k')}) / \tau } } \right]
\end{equation}
with temperature $\tau$ to optimize $f_\theta$ to predict $z$ similar to $z^{(k)}$ for elements specified by $y$ and negative samples $z^{(k')}$.
The set of negative samples $\mathcal{Z}' = \mathcal{Z} \setminus \{z^{(k)}\}$ consists of all known annotated semantics $\mathcal{Z}$ in the dataset except the current sample annotation $z^{(k)}$. We optimize over all $\mathcal{Z}'$ for every batch instead of randomly sampling negatives as the number of semantics are tractable. We note that the general objective \eqref{eq:cl_objective} is equivalent to the previously proposed cross-entropy over softmax normalized embedding similarity objective~\cite{li2022lseg} 
\begin{equation}
\label{eq:lseg_objective}
    \mathcal{L}_{CE} = \E \left[ - (c^{(k)})^T \log \sigma \left( sim(\hat{z}, z^{(k)}) / \tau \right) \right]
\end{equation}
with $c^{(k)}$ denoting one-hot class or description type vectors, $\sigma()$ as the softmax function. The equivalence is apparent by zeroing out all but the one-hot true target embedding resulting from the dot product sum and expanding the softmax function
\begin{equation}
    \mathcal{L}_{CE} = \E \left[ 0 - \ldots - \log \frac{e^{ sim(\hat{z}, z^{(k)} ) / \tau}}{ \sum^K_{k'=1} e^{ sim(\hat{z}, z^{(k')}) / \tau } } - \ldots - 0 \right].
\end{equation}

Next we verify that the objective \eqref{eq:cl_objective}, and equivalently \eqref{eq:lseg_objective}, can learn latent compositional semantic embeddings $z^*$ from independent nonoverlapping descriptions. Proposition~\ref{proposition:gradient_descent} proves that $z^*$ can be learned by gradient descent. We can therefore presume without loss of generality, that two descriptions $z^{(k_1)}$ and $z^{(k_2)}$ appear simultaneously in a batch for two independent but visually similar objects $x_1$ and $x_2$ mapping to the same latent semantic $z$. The combined loss is
\begin{equation}
\label{eq:compositional_semantics_by_contrastive_learning}
\begin{split}
    \mathcal{L} &= \frac{1}{2} \left( \mathcal{L}_{CL}(z, z^{(k_1)}) + \mathcal{L}_{CL}(z, z^{(k_2)}) \right) \\
    &= \frac{1}{2} \left( - \log \frac{1}{c} e^{sim(z, z^{(k_1)})} - \log \frac{1}{c} e^{sim(z, z^{(k_2)})} \right) \\
    &= - \frac{1}{2} \left( \log e^{sim(z, z^{(k_1)})} + \log e^{sim(z, z^{(k_2)})} - 2 \log c \right) \\
    &= - \frac{1}{2} \left( sim(z, z^{(k_1)}) + sim(z, z^{(k_2)}) \right) + \log c 
\end{split}
\end{equation}

As the optimal $z$ minimizing \eqref{eq:compositional_semantics_by_contrastive_learning} equals the centroid of $z^{(k_1)}$ and $z^{(k_2)}$, the optimal $z$ is the optimal latent compositional semantic embedding $z^*$ as proved by Theorem~\ref{theorem:discoverability}. We conclude that the iterative optimization by objective \eqref{eq:cl_objective} enable $f_\theta$ to learn $z^*$ from visual similarity and nonoverlapping descriptions.

\subsection{Sufficient similarity inference method}
\label{sec:sufficient_similiarty}

Conventional semantic segmentation presume an input image can be sensibly partitioned into a set of $K$ fixed hand-crafted semantic classes $\mathcal{E}_K$. Each class $k$ is represented by a one-hot embedding $e^{(k)} \in \mathcal{E}_K$. The embeddings $\mathcal{E}_K$ span different dimensional axes on the positive quadrant of the unit hypersphere $S^{K-1}$. The partitioning is computed by assigning class $k^*$ represented by the most similar embedding $e^{(k)}$ to each predicted embedding $\hat{z}$
\begin{equation}
\label{eq:most_similar_eval}
    k^* = \argmax_{k} \left[ sim(\hat{z}, e^{(k)}) \right] \;\; \forall e^{(k)} \in \mathcal{E}_K.
\end{equation}
Open world semantic segmentation likewise partition the image by assigning the most similar semantic $k^*$ in a set of word semantics $\mathcal{Z}_K$ distributed over $S^{K-1}$. The semantics of $\mathcal{E}_K$ defines the orthogonal basis of $S^{K-1}$ and thus limit queryable semantics to $\mathcal{E}_K$. In contrast, learning word semantics results in a semantically meaningful orthogonal basis, allowing any $\mathcal{Z}_K$ to be defined and queried at inference time.

Boyi~et~al.\cite{li2022lseg} identifies two weaknesses of the most similar partitioning approach: First, any object such as a \textit{window-on-a-building-facade} can both be described as a ``window'' as well as part of a ``building'' at a higher-level. Hard partitioning by highest similarity haphazardly predicts one or the other. Secondly, hard partitioning assigns a semantic to every image element even if all queried semantics have low similarity with the image content. An example is a \text{dog} queried by the two semantics ``grass'' and ``toy'' is interpreted as ``toy''. The use of abstract word semantics like ``other'' as a substitute for unspecified semantics is not a principled solution as there is no guarantee that the similarity between $z^*$ and queried but unrelated semantic $z^{(k)}$ is less similar than the ambiguous semantic meaning of ``other''
\begin{equation}
    sim(z^*, z_{other}) \overset{?}{>} sim(z^*, z^{(k)}) \;\; \forall z^{(k)} \in \mathcal{Z}_K.
\end{equation}

We propose sufficient similarity as a principled inference method that allows semantic overlap and empty query results by a single compositional semantic embedding $z^*$.
To evaluate semantic membership by sufficient similarity, we first compute a set of similarity threshold values $T~=~\{\tau_1, \ldots, \tau_K\}$ for each known semantic $k \in \{1, \ldots, K\}$. The value of $\tau_k$ is found by maximizing the likelihood that $sim(z^*, z^{(k)}) > \tau_k$ for true elements in past observations.
At evaluation time, instead of selecting the most similar semantic $k^*$ in \eqref{eq:most_similar_eval}, any similarity with semantic $z^{(k)}$ higher than the threshold $\tau_k$ are deemed sufficiently similar to be a member of the semantic group $k$
\begin{equation}
\label{eq:sufficient_semantic_similarity}
    sim(z^*, z^{(k)}) > \tau_k \Rightarrow MemberOf(z^*, k).
\end{equation}

We view \eqref{eq:sufficient_semantic_similarity} as a practical probabilistic approach for finding the mathematically derived hyperspherical cap $S^{D-1}_{cap}$~\eqref{eq:hyperspherical_cap} defining the membership set $\mathcal{Z}$~\eqref{eq:compositional_semantics_membership_set} that maximizes the likelihood over past observations.
For simplicity, we estimate a single maximum likelihood value $\tau_k$ for each semantic $k$ by a logistic regression model. To fit the model, a set of similarity values $sim(z^*, z)$ are sampled from positive and negative elements of $k$ using annotations $y$. The optimal $\tau_k$ is the the decision boundary or $sim(z^*, z)$ value that best separates positive and negative elements according to the model
\begin{equation}
    p \Bigl(MemberOf\left( sim(z^*, z) , k  \right)  \Bigl) = 0.5.
\end{equation}
However, our method is not fundamentally limited to estimating only single constant values $\tau_k$.
To the best of our knowledge, the similarity thresholding method proposed by Cui~et~al.~\cite{cui2020sim_thresh} is closest to our approach. While Cui~et~al. uses thresholding for uncertainty estimation, we propose thresholding to determine category membership~\eqref{eq:membership_by_sim}.

%


\section{Experiments}
\label{sec:experiments}

In the following sections we set out to verify the properties and discoverability of latent compositional semantic embeddings $z^*$ derived in Sec.~\ref{sec:compsitional_semantics}~and~\ref{sec:dense_VLMs}.
We perform experiments on embedding spaces for four representative models: the VLMs CLIP~\cite{radford2021clip}, OpenCLIP~\cite{schuhmann2022openclip}, X-Decoder~\cite{zou2023x_decoder}, and the language model SBERT~\cite{reimers2019sbert}. Additionally, we do experiments on ideal uniformly distributed embedding spaces $\mathit{U}(S^{D-1})$.

The first set of experiments investigates the lower bound capacity for $z^*$ to represent an arbitrary set of $K$ randomly sampled VL embeddings.
The second experiment set analyzes the expected capacity for $z^*$ to represent realistic object descriptions consisting of $K$ semantic.
The third experiment set verifies that $z^*$ are discoverable from visual appearance and nonoverlapping semantic annotations when training open-vocabulary unconditional dense VLMs.

Our experiments and results are relevant for any spatio-semantic representation method based on projecting image semantics to 3D or 2D coordinates~\cite{huang2023vlmaps, karlsson2023pred_wm}. We propose evaluation using a general-purpose image dataset such as COCO~\cite{caesar2018coco_stuff} is a more useful performance indicator than narrow 3D datasets. The expected semantic prediction performance for image-frame and projected semantic point cloud inference is verified to be equivalent~\cite{karlsson2023pred_wm}.


\subsection{Experiment 1: $z^*$ from random semantics}
\label{sec:exp_1_random_semantics}

We estimate the lower bound capacity of $z^*$ by sampling $K$ embeddings $z$ forming an object description set $\mathcal{Z}$ of random semantics. Next we compute the optimal $z^*$ by~\eqref{eq:optimal_comp_sem_centroid} and measure the separation between 100,000 randomly sampled embeddings $Z'$ and the set $Z$ represented by $z^*$. Separability is measured by~\eqref{eq:statistical_satisfiability} approximating~\eqref{eq:prob_of_compositionality} for uniform and nonuniform distributions. High separability means it is highly unlikely any non-related random semantic is closer to $z^*$ than the least close related semantic $z_{min} = \argmin(Z)$. In other words, $z^*$ has high cosine similarity~\eqref{eq:cos_sim} only with semantics $z$ of the object description $\mathcal{Z}$. See Fig.~\ref{fig:compositional_semantics} for a visualization.

To generate embeddings, we sample words from the English lexical database WordNet~\cite{miller1995wordnet}. Sampled words gets transformed into a semantic embedding $z$ by the models' language encoders. Ideally distributed embeddings are sampled uniformly on the hypersphere $\mathit{U}(S^{D-1})$.
CLIP experiments use the largest available \textsl{ViT-L/14@336px} model generating 768 dimensional embeddings. For OpenCLIP we use the largest \textsl{ViT-bigG-14} model, pretrained on the \textsl{laion2b\_s39b\_b160k} dataset, generating 1280 dimensional embeddings. We use the largest available \textsl{Focal-L} model for X-Decoder outputting 512 dimensional embeddings. SBERT uses the \textsl{all-mpnet-base-v2 } model generating 768 dimensional embeddings.
We measure performance of object descriptions $\mathcal{Z}$ of varying length $K$ to estimate maximum representation capacity of $z^*$ for each embedding space. Two additional experiments for higher dimensional embeddings explore the theoretical limits of $z^*$ for large object descriptions $\mathcal{Z}$. Each experiment is repeated 1000 times for statistical estimation.

\subsection{Experiment 2: $z^*$ from object descriptions}
\label{sec:exp_2_realistic_descriptions}

The second set of experiments estimates the separability for 500 realistic object descriptions consisting of related semantics. Each object description is generated by an LLM~\footnote{Claude 2 provided by Anthropic (\url{claude.ai})} and consists of $K$ descriptive semantics including names, properties, and affordances. The results represent expected representational capacity of $z^*$ in practical real-world application.
%


\begin{table}[t]
\begin{center}
\caption{Compositional semantics expectation delta}
\label{tab:exp_1_expectation_delta}
\begin{tabular}{c c c c}
                        & \multicolumn{3}{c}{$\Delta \E = \E sim(z^*, z) - \E sim(z^*, z')$}                                                                             \\ \hline
\multicolumn{1}{|l|}{Distribution} & \multicolumn{1}{c|}{$K=3$} & \multicolumn{1}{c|}{$K=5$} & \multicolumn{1}{c|}{$K=10$} \\ \hline
\multicolumn{1}{|l|}{CLIP~\cite{radford2021clip} $^{\text{b}}$} & \multicolumn{1}{l|}{0.135 (0.043)} & \multicolumn{1}{l|}{0.083 (0.032)} & \multicolumn{1}{l|}{0.043 (0.024)} \\ \hline
\multicolumn{1}{|l|}{OpenCLIP~\cite{schuhmann2022openclip} $^{\text{c}}$} & \multicolumn{1}{l|}{0.245 (0.032)} & \multicolumn{1}{l|}{0.156 (0.027)} & \multicolumn{1}{l|}{0.083 (0.020)} \\ \hline
\multicolumn{1}{|l|}{X-Decoder~\cite{zou2023x_decoder} $^{\text{a}}$} & \multicolumn{1}{l|}{0.236 (0.045)} & \multicolumn{1}{l|}{0.150 (0.037)} & \multicolumn{1}{l|}{0.080 (0.027)} \\ \hline
\multicolumn{1}{|l|}{SBERT~\cite{reimers2019sbert} $^{\text{b}}$} & \multicolumn{1}{l|}{0.397 (0.040)} & \multicolumn{1}{l|}{0.273 (0.035)} & \multicolumn{1}{l|}{0.156 (0.026)} \\ \hline
\multicolumn{1}{|l|}{$\mathit{U}(z)_{D = 768}$} & \multicolumn{1}{l|}{0.577 (0.012)} & \multicolumn{1}{l|}{0.447 (0.010)} & \multicolumn{1}{l|}{0.316 (0.080)} \\ \hline
\multicolumn{1}{|l|}{$\mathit{U}(z)_{D = 1280}$} & \multicolumn{1}{l|}{0.577 (0.010)} & \multicolumn{1}{l|}{0.447 (0.080)} & \multicolumn{1}{l|}{0.316 (0.006)} \\ \hline
\multicolumn{1}{|l|}{$\mathit{U}(z)_{D = 2048}$} & \multicolumn{1}{l|}{0.577 (0.008)} & \multicolumn{1}{l|}{0.447 (0.006)} & \multicolumn{1}{l|}{0.316 (0.005)} \\ \hline
\multicolumn{1}{|l|}{$\mathit{U}(z)_{D = 4096}$} & \multicolumn{1}{l|}{0.577 (0.005)} & \multicolumn{1}{l|}{0.447 (0.005)} & \multicolumn{1}{l|}{0.316 (0.003)} \\ \hline
\multicolumn{4}{l}{a: $D = 512$, b: $D = 768$, c: $D = 1280$}
\end{tabular}
\end{center}
\vspace{-5mm}
\end{table}

\begin{table}[t]
\begin{center}
\caption{Separation of related and nonrelated random semantics}
\label{tab:exp_1_separation}
\begin{tabular}{c c c c}
                        & \multicolumn{3}{c}{$P\left(sim(z^*, z) > sim(z^*, z')\right)$}                                                                             \\ \hline
\multicolumn{1}{|l|}{Distribution} & \multicolumn{1}{c|}{$K=3$} & \multicolumn{1}{c|}{$K=5$} & \multicolumn{1}{c|}{$K=10$} \\ \hline
\multicolumn{1}{|l|}{CLIP~\cite{radford2021clip} $^{\text{b}}$} & \multicolumn{1}{l|}{0.954 (0.117)} & \multicolumn{1}{l|}{0.533 (0.261)} & \multicolumn{1}{l|}{0.187 (0.143)} \\ \hline
\multicolumn{1}{|l|}{OpenCLIP~\cite{schuhmann2022openclip} $^{\text{c}}$} & \multicolumn{1}{l|}{1.000 (0.001)} & \multicolumn{1}{l|}{0.907 (0.115)} & \multicolumn{1}{l|}{0.400 (0.180)} \\ \hline
\multicolumn{1}{|l|}{X-Decoder~\cite{zou2023x_decoder} $^{\text{a}}$} & \multicolumn{1}{l|}{0.990 (0.0223)} & \multicolumn{1}{l|}{0.750 (0.1682)} & \multicolumn{1}{l|}{0.301 (0.156)} \\ \hline
\multicolumn{1}{|l|}{SBERT~\cite{reimers2019sbert} $^{\text{b}}$} & \multicolumn{1}{l|}{1.000 (0.002)} & \multicolumn{1}{l|}{0.977 (0.043)} & \multicolumn{1}{l|}{0.647 (0.188)} \\ \hline
\multicolumn{1}{|l|}{$\mathit{U}(z)_{D = 768}$} & \multicolumn{1}{l|}{1 (0)} & \multicolumn{1}{l|}{1 (0)} & \multicolumn{1}{l|}{1 (0)} \\ \hline
\multicolumn{1}{|l|}{$\mathit{U}(z)_{D = 1280}$} & \multicolumn{1}{l|}{1 (0)} & \multicolumn{1}{l|}{1 (0)} & \multicolumn{1}{l|}{1 (0)} \\ \hline
\multicolumn{1}{|l|}{$\mathit{U}(z)_{D = 2048}$} & \multicolumn{1}{l|}{1 (0)} & \multicolumn{1}{l|}{1 (0)} & \multicolumn{1}{l|}{1 (0)} \\ \hline
\multicolumn{1}{|l|}{$\mathit{U}(z)_{D = 4096}$} & \multicolumn{1}{l|}{1 (0)} & \multicolumn{1}{l|}{1 (0)} & \multicolumn{1}{l|}{1 (0)} \\ \hline
\multicolumn{4}{l}{a: $D = 512$, b: $D = 768$, c: $D = 1280$}
\end{tabular}
\end{center}
\vspace{-5mm}
\end{table}

\begin{table}[t!]
\begin{center}
\caption{Large object description expectation delta and separation}
\label{tab:exp_1_large_dim}
\begin{tabular}{c c c c}
                                        & \multicolumn{3}{c}{$K = 100$}                                                          \\ \hline
\multicolumn{1}{|l|}{\multirow{2}{*}{Distribution}} & \multicolumn{1}{c|}{\multirow{2}{*}{$\Delta \E$}} & \multicolumn{2}{c|}{$P\left(sim(z^*, z) > sim(z^*, z')\right)$}                         \\ \cline{3-4} 
\multicolumn{1}{|l|}{} & \multicolumn{1}{c|}{}                            & \multicolumn{1}{c|}{Empirical~\eqref{eq:statistical_satisfiability}} & \multicolumn{1}{c|}{Bound~\eqref{eq:prob_of_compositionality}} \\ \hline
\multicolumn{1}{|l|}{CLIP~\cite{radford2021clip} $^{\text{b}}$}           & \multicolumn{1}{l|}{0.004 (0.008)} & \multicolumn{1}{l|}{0.011 (0.011)} & \multicolumn{1}{l|}{1.0$^\star$} \\ \hline
\multicolumn{1}{|l|}{OpenCLIP~\cite{schuhmann2022openclip} $^{\text{c}}$} & \multicolumn{1}{l|}{0.009 (0.007)} & \multicolumn{1}{l|}{0.013 (0.012)} & \multicolumn{1}{l|}{1.0$^\star$} \\ \hline
\multicolumn{1}{|l|}{X-Decoder~\cite{zou2023x_decoder} $^{\text{a}}$} & \multicolumn{1}{l|}{0.008 (0.009)} & \multicolumn{1}{l|}{0.013 (0.012)} & \multicolumn{1}{l|}{1.0$^\star$} \\ \hline
\multicolumn{1}{|l|}{SBERT~\cite{reimers2019sbert} $^{\text{b}}$} & \multicolumn{1}{l|}{0.018 (0.009)} & \multicolumn{1}{l|}{0.015 (0.013)} & \multicolumn{1}{l|}{1.0$^\star$} \\ \hline
\multicolumn{1}{|l|}{$\mathit{U}(z)_{D = 768}$}                           & \multicolumn{1}{l|}{0.010 (0.001)} & \multicolumn{1}{l|}{0.605 (0.148)} & \multicolumn{1}{l|}{0.612} \\ \hline
\multicolumn{1}{|l|}{$\mathit{U}(z)_{D = 1280}$}                          & \multicolumn{1}{l|}{0.010 (0.002)} & \multicolumn{1}{l|}{0.838 (0.116)} & \multicolumn{1}{l|}{0.863} \\ \hline
\multicolumn{1}{|l|}{$\mathit{U}(z)_{D = 2048}$}                          & \multicolumn{1}{l|}{0.010 (0.002)} & \multicolumn{1}{l|}{0.967 (0.040)} & \multicolumn{1}{l|}{0.988} \\ \hline
\multicolumn{1}{|l|}{$\mathit{U}(z)_{D = 4096}$}                          & \multicolumn{1}{l|}{0.010 (0.001)} & \multicolumn{1}{l|}{1.000 (0.001)} & \multicolumn{1}{l|}{1.000} \\ \hline
\multicolumn{4}{l}{a: $D = 512$, b: $D = 768$, c: $D = 1280$, $^\star$: Error from non-uniformity}
\end{tabular}
\end{center}
\vspace{-5mm}
\end{table}

\subsection{Experiment 3: $z^*$ from visual appearance}
\label{sec:exp_3_visual_appearance}

The third experiment set investigates if $z^*$ can be discovered from independent observations of visual appearance paired with nonoverlapping annotations.
We present two experiments to answer this question.

First, we evaluate how well four representative SOTA unconditional open vocabulary semantic segmentation models can infer overlapping compositional semantics. Each model is trained on conventional non-overlapping annotations.
ZSSeg~\cite{xu2021zsseg} generates region proposals by SAM~\cite{kirillov2023sam} and uses CLIP~\cite{radford2021clip} to predict semantic embeddings $z$.
X-Decoder~\cite{zou2023x_decoder} is a conditional VLM that predicts $N$ object mask proposals and match masks with the most likely query semantic. We convert X-Decoder into an unconditional model by integrating all $N$ VL mask semantics by the mask probability at each pixel location. We use largest available \textsl{Focal-L} model trained on COCO captions and dense labels.
LSeg~\cite{li2022lseg} is a dense VLM trained to output unconditional VL embedding maps. We use the released \textsl{ViT-L/16} model weights trained on seven datasets including COCO-Stuff~\cite{caesar2018coco_stuff}, ADE20K~\cite{zhou2017ade20k}, and Mapillary~\cite{neuhold2017mapillary_vistas}.
ViT-Adapter~\cite{chen2022vitadapter} is a recent general-purpose dense computer vision architecture we implement as our trainable model. The ViT backbone is initialized with BEiT model weights~\cite{bao2022beit}. The model is trained with SBERT embeddings on the same seven dataset as LSeg for 160K iterations on four A6000 GPUs with a total batch size 4 and $0.75\text{e-}4$ learning rate.
We create three modified datasets with overlapping semantics following the three level label hierarchy proposed in the COCO-Stuff dataset~\cite{caesar2018coco_stuff} (e.g. a \textit{car-object} is described as either ``car'', ``vehicle'', or ``outdoor''). We emphasize that none of the models have been explicitly trained on the additional overlapping semantics.
See the supplementary material for further details.

Secondly, we estimate the performance gained by directly training a model with overlapping annotations on the COCO-Stuff dataset~\cite{caesar2018coco_stuff} as an upper performance bound. We train four ViT-Adapter models using CLIP or SBERT embeddings with two dataset variants. The first variant uniformly samples annotations from one of the three label hierarchy levels. The second variant weights sampling so all annotation classes are equally likely. Uniform and weighted sampling represent the long-tail distribution over low- and high-level semantics, respectively. Each image annotation is sampled only once for each sample, meaning compositional semantics must be learned by generalizing from independent observations of visual appearance.
Additionally, we estimate separability \eqref{eq:statistical_satisfiability} and distance between the learned $z^*$ embeddings with the optimal $z^*_{opt}$ computed as the centroid of the ground truth overlapping semantics~\eqref{eq:optimal_comp_sem_centroid}.

We evaluate compositional semantics by mIoU computed using the conventional \textsl{most similar} partitioning and our proposed \textsl{sufficient similarity} method introduced in Sec.~\ref{sec:sufficient_similiarty}. To use sufficient similarity we precompute $\tau_k$ for every semantic category $k$ from 2000 samples from the training dataset covering all annotation semantics.

We emphasize that our study concerns the theoretical understanding and practical feasibility of learning compositional semantics across a variety of models. We implement the trainable ViT-Adapter based model to investigate the highest achievable SOTA performance with existing model architectures, but do not consider it our core research contribution.


\section{Results}
\label{sec:results}

Here we present results demonstrating the capacity of latent compositions embeddings $z^*$ for representing rich object descriptions in different embedding spaces. We also demonstrate that $z^*$ can be learned by unconditional dense VLM from nonoverlapping annotations.

\subsection{Experimental results 1: $z^*$ from random semantics}
\label{sec:results_1_random_semantics}
Here we provide results and findings for $z^*$ representing sets $\mathcal{Z}$ of randomly sampled semantics $z$.
Table~\ref{tab:exp_1_expectation_delta} shows the expected similarity between optimal $z^*$ \eqref{eq:optimal_comp_sem_centroid} and object description semantics $z$ is always higher than for unrelated semantics $z'$. The results verifies that $z^*$ for all embedding distributions and object description sizes $K$ satisfy Definition~\ref{def:compositionality} and Theorem~\ref{theorem:discoverability} for finding the optimal $z^*$.
Table~\ref{tab:exp_1_separation} shows lower bound separability of related $z \in \mathcal{Z}$ and non-related semantics $z' \in \mathcal{Z'}$ by $z^*$, verifying Theorem~\ref{theorem:prob_of_compositionality}. All embedding spaces allow reliable separability for small object descriptions $K \le 3$, verifying Proposition~\ref{proposition:discoverability_2} for finding $z^*$ for non-uniform distributions. For intermediate descriptions $K \le 5$ separability of CLIP embeddings reduces to chance. SBERT maintains strong separability. Only ideal uniform distributions achieve perfect separability for large descriptions $K \le 10$.
Table~\ref{tab:exp_1_large_dim} shows that sufficiently high-dimensional uniformly distributed embedding spaces can represent very large object descriptions of size $K \le 100$ with perfect separability. Note that largest 4096 dimension embedding space equals the ResNet output embedding map dimension~\cite{he2016resnet}. The probabilistic bound~\eqref{eq:prob_of_compositionality} accurately predict the empirical separation probability result for uniform distributions. The bound fails for highly non-uniform distributions as expected.
Figure~\ref{fig:exp_1} visualizes embedding similarity distributions for different embedding spaces and object descriptions sizes $K$.
Figure~\ref{fig:exp_1_high-dim} shows how increasing dimensionality gradually improves separability.

We find that the object description size $K$ representable by $z^*$ is only constrained by embedding space dimensionality $D$ and degree of uniformity. The OpenCLIP embedding space provides better separability than the popular CLIP and SOTA multi-task optimized X-Decoder models. The pure language model SBERT has better embedding space than all VLM models. We propose to learn unconditional dense VLMs on language model embeddings instead of global description VLMs like CLIP as the pretrained vision encoder is not used . The findings motivate further work towards increasing uniformity of existing VLM embedding distributions to better leverage the capacity of high-dimensional embedding spaces and to improve discriminatability of compositional semantic embeddings~\cite{Li2011HypersphereCap, so2022geodesic}.

\begin{figure*}
\centering
\includegraphics[width=1.\textwidth] {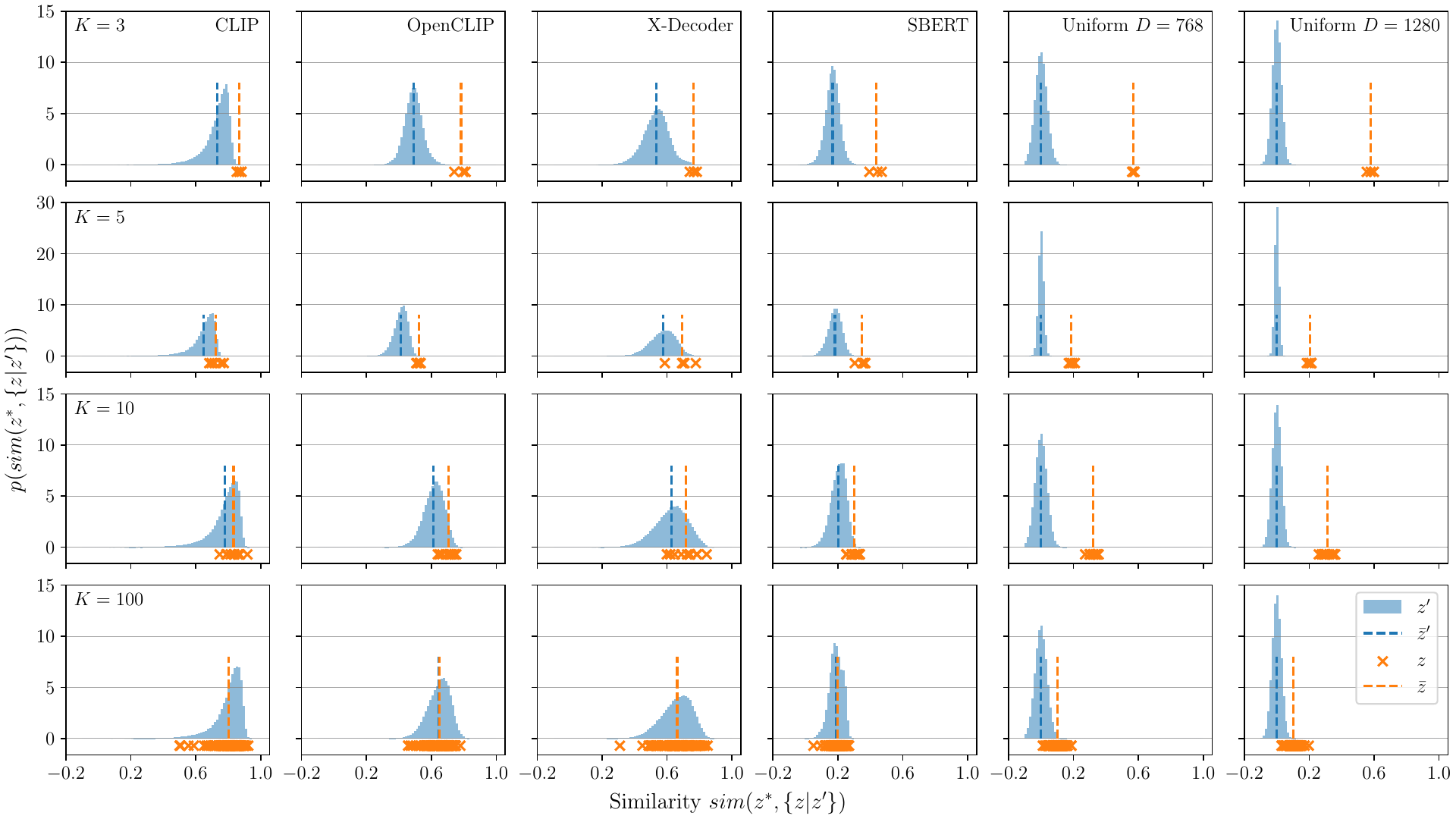}
\vspace{-8mm}
\caption{Similarity distributions between a latent compositional semantic embedding $z^*$ and all object description embeddings $z \in \mathcal{Z}$ it represent (orange) and randomly sampled unrelated word embeddings $z'$ (blue). Columns show different embedding spaces. Each row shows object descriptions of different size $K$. A $z^*$ is useful if it separates the distribution of $z$ and $z'$ by cosine similarity~\eqref{eq:cos_sim}.}
\label{fig:exp_1}
\vspace{-5mm}
\end{figure*}




\subsection{Experimental results 2: $z^*$ from object descriptions}
\label{sec:results_2_results_object_descriptions}
Here we provide separability results for $z^*$ representing sets $\mathcal{Z}$ of realistic object descriptions composed of related semantics $z$.
Table~\ref{tab:exp_1_separation} shows that realistic sets of related semantics have better separability than the lower bound of random semantic descriptions presented in Table~\ref{tab:exp_1_separation}. All VLMs achieve strong separability for $K \le 5$, and SBERT allows large object representations of $K \le 10$.


Figure~\ref{fig:exp_2} visualizes similarity distributions for three particular object descriptions of varying lengths $K$. 
The top row shows distributions for the short object description of a ``medium-sized utility vehicle'' $\mathcal{Z}_1 = $~\{truck, van, vehicle\}. All related $z \in \mathcal{Z}_1$ are perfectly separable from the distribution of non-related $z' \notin \mathcal{Z}_1$ by $z^*$ and $\theta_z$ given by \eqref{eq:theta_min}.
The middle row shows the separability of a medium sized description for a ``patch on a drivable flat asphalt road with painted lane markings'' $\mathcal{Z}_2 = $~\{road, lane marking, drivable, asphalt, flat\}. All models achieve above 99 \% separability.
The bottom row visualizes the distribution of a large description of a ``white wooden table surface'' $\mathcal{Z}_3 = $~\{ 
'table', 'wood', 'counter', 'solid', 'surface', 'white', 'static', 'flat', 'furniture', 'static' \}. The VLMs do not reach reliable separability. In contrast, the language model SBERT achieves 98 \% separability, demonstrating that SBERT embedding are practically useful up to about 10 semantics. We consider the results as upper bounds for visually learned representations by VLMs.

\begin{figure}
\centering
\includegraphics[width=0.48\textwidth] {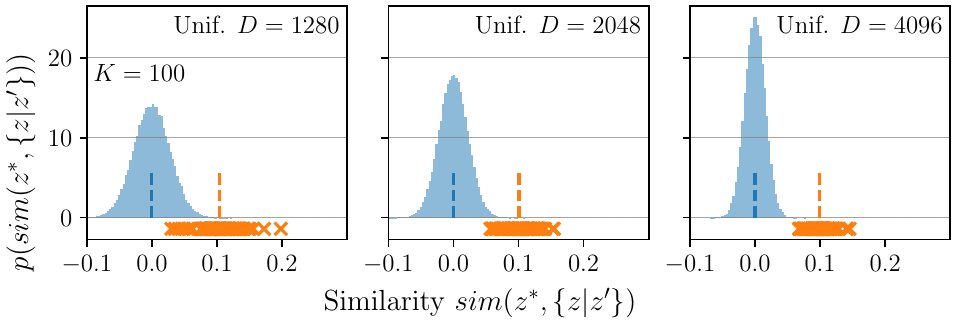}
\vspace{-4mm}
\caption{Similarity distributions for large object descriptions $\mathcal{Z}$ in very high-dimensional uniformly distributed embedding spaces.}
\label{fig:exp_1_high-dim}
\vspace{-5mm}
\end{figure}

\begin{figure}
\centering
\includegraphics[width=0.48\textwidth] {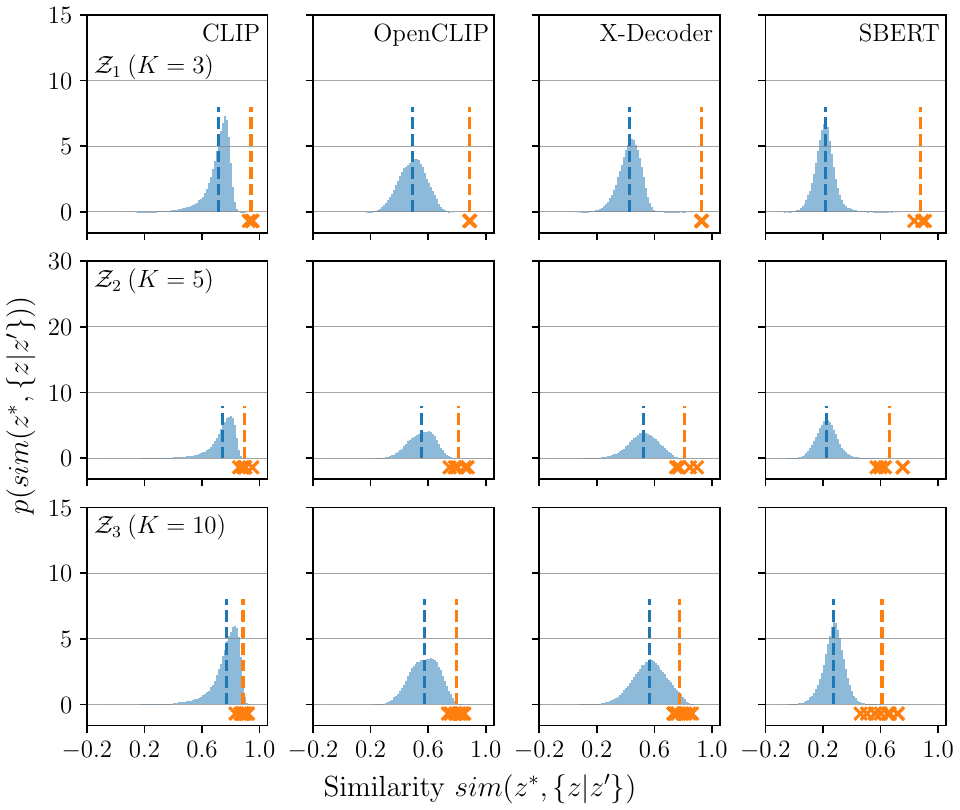}
\vspace{-4mm}
\caption{Similarity distributions for three realistic object descriptions $\mathcal{Z}_i$ of varying sizes $K$ (orange) and randomly sampled word embeddings $z'$ (blue).}
\label{fig:exp_2}
\vspace{-5mm}
\end{figure}




\subsection{Experimental results 3: $z^*$ from visual appearance}

Our results show common VLM models trained on conventional nonoverlapping annotations discover compositional semantics $z^*$ as specified by Definition~\ref{def:compositionality} and Theorem~\ref{theorem:prob_of_compositionality}. Discovering $z^*$ enables inferring overlapping semantics by our proposed \textsl{sufficient similarity} inference method (see Sec~\ref{sec:sufficient_similiarty}).

Table~\ref{tab:exp_3_conventional_models} presents segmentation performance for original non-overlapping annotations (e.g. COCO) and our novel compositional semantics (e.g. COCO CS) dataset variants with overlapping annotations. Each model is evaluated by the conventional most similar (MS) and our proposed sufficient similarity (SS) method. Levels (denoted $^{(2)}$) specify which hierarchical semantics are being evaluated (e.g. level 1 \textit{cat}, level 2 \textit{animal}, and level 3 \textit{outdoor}).
The region proposal method ZSeg~\cite{xu2021zsseg} underperforms other models explicitly trained on dense annotations despite the promise of highest generality.
The mask-based conditional method X-Decoder~\cite{zou2023x_decoder} modified to output unconditional dense embedding maps performs worse than the inherently unconditional pixel-level prediction model LSeg~\cite{li2022lseg}.
Our ViT-Adapter~\cite{chen2022vitadapter} based model implementation with a general-purpose SOTA architecture for dense vision tasks performs even better on both our novel overlapping and non-overlapping semantic inference tasks.
Our proposed sufficient similarity inference method improves inference performance of second level overlapping semantics across all models by 19.63 mIoU on average.
 Conventional most similar (MS) inference has a performance advantage over the sufficient similarity (SS) inference on level 1 semantics. The reason is that MS inference overfits level 1 semantics due to their prevalence in training data. Additionally, MS inference is fundamentally limited to predicting a single semantic, unlike SS inference which can theoretically achieve a perfect overlapping segmentation score.

\begin{table*}[]
\begin{center}
\caption{Unconditional open vocabulary segmentation and overlapping segmentation performance}
\vspace{-2mm}
\label{tab:exp_3_conventional_models}
\setlength{\tabcolsep}{0.5em}
\begin{tabular}{llllllllllllllll}
                                                         & \multicolumn{15}{c}{mIoU}                                                                                                                                                                                                                                                                                                                            \\ \hline
\multicolumn{1}{|l|}{\multirow{2}{*}{Model}}             & \multicolumn{1}{c|}{COCO}  & \multicolumn{4}{c|}{COCO CS}                                                      & \multicolumn{1}{c|}{ADE}   & \multicolumn{4}{c|}{ADE CS}                                                       & \multicolumn{1}{c|}{Mapillary} & \multicolumn{4}{c|}{Mapillary CS}                                                 \\
\multicolumn{1}{|l|}{}                                   & \multicolumn{1}{c|}{MS}    & \multicolumn{1}{c}{MS} & \multicolumn{1}{c|}{SS}    & \multicolumn{1}{c}{MS$^{(2)}$} & \multicolumn{1}{c|}{SS$^{(2)}$} & \multicolumn{1}{c|}{MS}    & \multicolumn{1}{c}{MS} & \multicolumn{1}{c|}{SS}    & \multicolumn{1}{c}{MS$^{(2)}$} & \multicolumn{1}{c|}{SS$^{(2)}$} & \multicolumn{1}{c|}{MS}        & \multicolumn{1}{c}{MS} & \multicolumn{1}{c|}{SS}    & \multicolumn{1}{c}{MS$^{(2)}$} & \multicolumn{1}{c|}{SS$^{(2)}$} \\ \hline
\multicolumn{1}{|l|}{ZSSeg~\cite{xu2021zsseg}}              & \multicolumn{1}{l|}{11.23}      &  10.87     & \multicolumn{1}{l|}{2.24}      &     3.21       & \multicolumn{1}{l|}{8.28}           & \multicolumn{1}{l|}{9.93}  & 9.14  & \multicolumn{1}{l|}{3.05}  &     6.35       & \multicolumn{1}{l|}{5.86}       & \multicolumn{1}{l|}{6.51}      & 5.39  & \multicolumn{1}{l|}{3.39}  &     0.67       & \multicolumn{1}{l|}{13.75}      \\ \hline
\multicolumn{1}{|l|}{X-Decoder~\cite{zou2023x_decoder}}    & \multicolumn{1}{l|}{28.57} & 25.33 & \multicolumn{1}{l|}{28.14} & 10.19      & \multicolumn{1}{l|}{22.52}      & \multicolumn{1}{l|}{6.48}  & 5.88  & \multicolumn{1}{l|}{12.77} & 4.85       & \multicolumn{1}{l|}{14.39}      & \multicolumn{1}{l|}{11.52}     & 9.07  & \multicolumn{1}{l|}{9.39}  & 2.03       & \multicolumn{1}{l|}{20.26}      \\ \hline
\multicolumn{1}{|l|}{LSeg~\cite{li2022lseg}}                & \multicolumn{1}{l|}{38.42} & 37.10 & \multicolumn{1}{l|}{20.41} & 14.59      & \multicolumn{1}{l|}{14.66}      & \multicolumn{1}{l|}{27.40} & 25.11 & \multicolumn{1}{l|}{11.37} & 6.35       & \multicolumn{1}{l|}{16.23}      & \multicolumn{1}{l|}{30.08}     & 24.51 & \multicolumn{1}{l|}{11.81} & 0.01       & \multicolumn{1}{l|}{22.97}      \\ \hline
\multicolumn{1}{|l|}{ViT-Adapter~\cite{chen2022vitadapter}} & \multicolumn{1}{l|}{\textbf{48.12}} & \textbf{46.97} & \multicolumn{1}{l|}{39.16} & 12.55      & \multicolumn{1}{l|}{\textbf{54.19}}      & \multicolumn{1}{l|}{\textbf{47.47}} & \textbf{43.21} & \multicolumn{1}{l|}{30.29} & 28.99      & \multicolumn{1}{l|}{\textbf{31.63}}      & \multicolumn{1}{l|}{\textbf{46.92}}     & \textbf{37.94} & \multicolumn{1}{l|}{24.69} & 0.00          & \multicolumn{1}{l|}{\textbf{40.15}}      \\ \hline
\multicolumn{16}{l}{MS: Most similar evaluation, SS: Sufficient similarity evaluation, $^{(2)}$: Level 2 semantics evaluation only, CS: Compositional semantics}                                                                                                                                                                                                                                               
\end{tabular}
\end{center}
\vspace{-5mm}
\end{table*}

\begin{table}[]
\begin{center}
\caption{Learning compositional semantics by overlapping annotations}
\vspace{-2mm}
\label{tab:exp_3_coco_compositional}
\setlength{\tabcolsep}{0.5em}
\begin{tabular}{llllllll}
\hline
\multicolumn{1}{|c|}{\multirow{2}{*}{Model}} & \multicolumn{1}{l|}{\multirow{2}{*}{p($\mathcal{D}$)}} & \multicolumn{6}{c|}{COCO CS {[}mIoU{]}}                                                                                                               \\
\multicolumn{1}{|c|}{}                       & \multicolumn{1}{l|}{}                      & \multicolumn{1}{c}{MS} & \multicolumn{1}{c|}{SS}    & MS$^{(2)}$ & \multicolumn{1}{l|}{SS$^{(2)}$} & \multicolumn{1}{c}{MS$^{(2,3)}$} & \multicolumn{1}{c|}{SS$^{(2,3)}$} \\ \hline
\multicolumn{1}{|l|}{\multirow{2}{*}{CLIP}}  & \multicolumn{1}{l|}{US}                    & 25.90                  & \multicolumn{1}{l|}{32.99} & 34.19 & \multicolumn{1}{l|}{57.29} & \textbf{33.27}                       & \multicolumn{1}{l|}{55.93}   \\ \cline{2-8} 
\multicolumn{1}{|l|}{}                       & \multicolumn{1}{l|}{WS}                    & \textbf{45.94}                  & \multicolumn{1}{l|}{37.89} & 12.71 & \multicolumn{1}{l|}{50.18} & 12.18                       & \multicolumn{1}{l|}{48.82}   \\ \hline
\multicolumn{1}{|l|}{\multirow{2}{*}{SBERT}} & \multicolumn{1}{l|}{US}                    & 24.95                  & \multicolumn{1}{l|}{38.19} & 33.92 & \multicolumn{1}{l|}{\textbf{58.55}} & 32.39                       & \multicolumn{1}{l|}{\textbf{57.38}}   \\ \cline{2-8} 
\multicolumn{1}{|l|}{}                       & \multicolumn{1}{l|}{WS}                    & 45.67                  & \multicolumn{1}{l|}{\textbf{42.23}} & 12.77 & \multicolumn{1}{l|}{56.55} & 12.26                       & \multicolumn{1}{l|}{55.57}   \\ \hline
\multicolumn{8}{l}{\begin{tabular}[c]{@{}l@{}}US: Uniform sampling, WS: Weighted sampling, $^{(2,3)}$ Level 2 and 3 se-\\mantics, SS: Sufficient similarity evaluation, MS: Most similar evaluation\end{tabular}}
\end{tabular}
\end{center}
\vspace{-8mm}
\end{table}


\begin{figure}[ht!]
\centering
\includegraphics[width=0.48\textwidth] {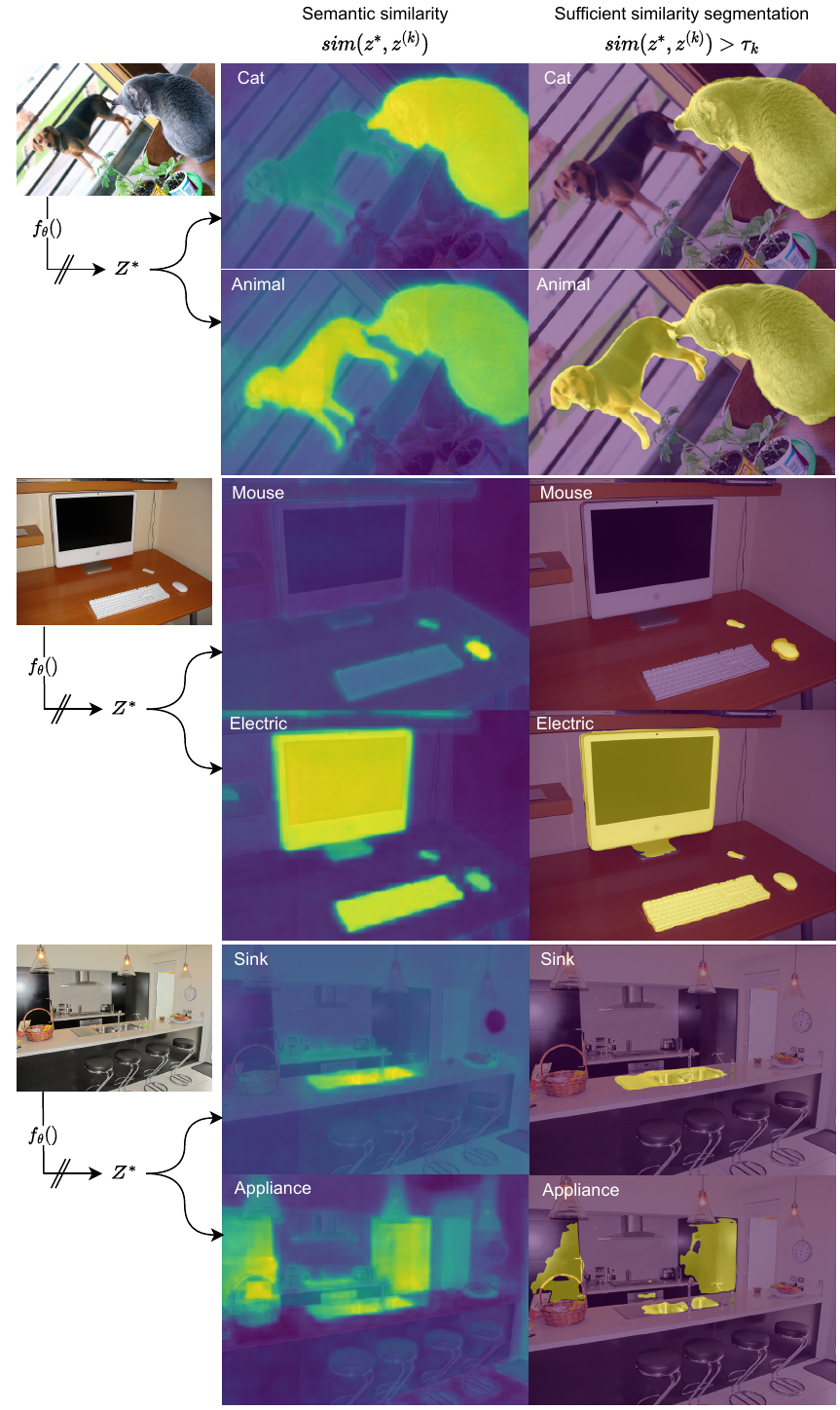}
\vspace{-3mm}
\caption{Examples of overlapping semantics inferrable from latent compositional semantic embeddings $z^*$ representing learned object descriptions $\mathcal{Z}$. The 3rd and 4th examples illustrate failure cases related to sufficient similarity threshold $\tau_k$ estimation for low- and high-level semantics, respectively.}
\label{fig:exp_3_viz}
\end{figure}

\begin{figure}[t!]
\centering
\includegraphics[width=0.48\textwidth] {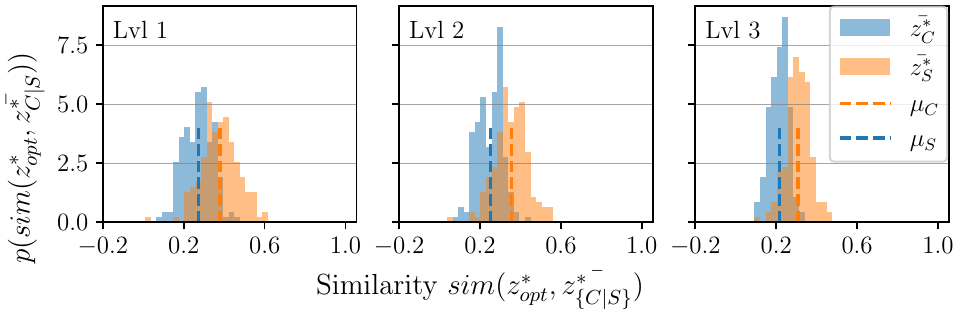}
\vspace{-4mm}
\caption{The distribution of mean similarities between optimal $z^*_{opt}$ and learned $z^*$ CLIP (blue) and SBERT (orange) embeddings for three semantic levels.}
\label{fig:exp_3_comp_sem_sim}
\vspace{-5mm}
\end{figure}


Table~\ref{tab:exp_3_coco_compositional} shows evaluation results for ViT-Adapter~\cite{chen2022vitadapter} models trained directly on overlapping COCO compositional semantics using different embedding spaces and annotation sampling strategies as described in Sec.~\ref{sec:exp_3_visual_appearance}.
Learning from a weighted sampling (WS) annotation distribution results in a uniform exposure of semantics from all levels and the best overall performance despite rarity of higher-level semantics.
The performance gap between the best ViT-Adapter model in Table~\ref{tab:exp_3_conventional_models} and Table~\ref{tab:exp_3_conventional_models} on overlapping conventional semantics is only 2.63 mIoU. The small gap indicate that learning $z^*$ from existing single non-overlapping annotations is an effective approach.
See Figure~\ref{fig:font_fig}~and~\ref{fig:exp_3_viz} for overlapping semantic inference visualizations.

%

%

In Figure~\ref{fig:exp_3_comp_sem_sim} we visualize the mean similarity distribution between learned $z^*$ and optimal $z^*_{opt}$ embeddings by Theorem~\ref{theorem:discoverability}.
Learned $z^*$ are far from optimal $z^*_{opt}$ for both CLIP and SBERT embedding models, similarly to how learned VL embeddings have a similarity or alignment gap with the encoded text annotations~\cite{wang2020hypersphere, wang2021understanding_cl}. However, the results in Table~\ref{tab:exp_3_conventional_models}-\ref{tab:exp_3_coco_compositional} proves that learned $z^*$ have adequate similarity with $z^*_{opt}$ for sufficient similarity segmentation of small semantic sets $\mathcal{Z}$. Increasing alignment between learned $z^*$ and $z^*_{opt}$ will enable $z^*$ to represent larger $\mathcal{Z}$ and approach the theoretical capacity of the text embedding space investigated in Sec~\ref{sec:results_1_random_semantics}-\ref{sec:results_2_results_object_descriptions}).






\section{Conclusions}
\label{sec:conclusions}
In this paper, we present a mathematical analysis and experimental verification of latent compositional semantic embeddings $z^*$ as a learnable knowledge representation for rich object descriptions satisfying the requirements of spatio-semantic representations.
Our VLM method is limited by the similarity gap between predicted $z^*$ and optimal $z^*_{opt}$.
For future work, we propose a dynamic approach to infer sufficient similarity thresholds $\tau_k$ which takes into account environment and task context. Additionally, we propose a learning objective optimizing for absolute similarity or greater alignment in addition to relative similarity~\eqref{eq:cl_objective}.
We hope our findings will contribute towards greater understanding of what semantics unconditional dense VLMs learn, and the theoretical limit of semantic expressivity of learned representations. 
Additionally, we hope to encourage more work towards learning and inferring compositional and overlapping semantics, and provide motivation for improving uniformity of VL embedding spaces.


\section*{Acknowledgments}
This work was financially supported by JST SPRING, Grant Number JPMJSP2125. The authors would like to take this opportunity to thank the ``Interdisciplinary Frontier Next-Generation Researcher Program of the Tokai Higher Education and Research System''.

The authors would like to thank Minh-Quan Dao, Yingjie Niu, Keisuke Fujii, and Kento Ohtani for proofreading and constructive criticism of the manuscript.


{\appendix[Mathematical proofs]
Here we provide full mathematical proofs for all theoreoms, propositions, and lemmas. 


\subsection{Proof for Lemma 1}
\label{app:lemma_1}

\begin{proof}
All normalized semantic embeddings $z$ are vectors in the set of vectors constituting the unit hypersphere
\begin{equation}
\label{eq:set_of_zs}
    z \in S^{D-1} = \{ z \in \mathbb{R}^{D} : \|z\| = 1 \}.
\end{equation}

The distribution of uniformly sampled random vectors $Z \sim \mathit{U}(S^{D-1})$ is isotropic (i.e. properties rotationally invariant). The covariance matrix $\Sigma$ of isotropic distributions equals the diagonal matrix $I_D$:
\begin{equation}
\label{eq:isotropy}
    \Sigma(Z) = \E Z Z^T = I_D.
\end{equation}

For the expected inner product of two independent random vectors $Z^{(i)}$, $Z^{(j)}$ sampled from an isotropic distribution it follows
\begin{equation}
    \E \inner{Z^{(i)}}{Z^{(j)}}^2 = \E_{Z^{(j)}} \E_{Z^{(i)}} \left[ \inner{Z^{(i)}}{Z^{(j)}}^2 | Z^{(j)}. \right]
\end{equation}

Assuming a particular but arbitrary vector $z^{(j)}$ and substituting \eqref{eq:isotropy} the inner expectation becomes
\begin{equation}
\label{eq:inner_expecation}
\begin{split}
    \E_{Z^{(i)}} \inner{Z^{(i)}}{z^{(j)}}^2 &= z^{(j)T} \E \left[ Z^{(i)} Z^{(i)T} \right] z^{(j)} \\
    &= z^{(j)T} I_D z^{(j)} \\
    &= z^{(j)T} z^{(j)} \\
    &= \|z^{(j)} \|^2
\end{split}
\end{equation}
The outer expectation after substituting \eqref{eq:inner_expecation} and \eqref{eq:isotropy} becomes
\begin{equation}
\begin{split}
    \E_{Z^{(i)}} \inner{Z^{(i)}}{z^{(j)}}^2 &= \E_{Z^{(j)}} \| z^{(j)} \|^2 = \E Z^{(j)T} Z^{(j)} \\
    &= \E tr \left[ Z^{(j)T} Z^{(j)} \right] \\
    &= \E tr \left[ Z^{(j)} Z^{(j)T} \right] \\
    &= tr \left[ \E Z^{(j)} Z^{(j)T} \right] \\
    &= tr \left[ I_D \right] = D.
\end{split}
\end{equation}
Expanding the inner product of two normalized random Euclidean vectors $\hat{Z}^{(i)}$, $\hat{Z}^{(j)}$ sampled from an isotropic distribution
\begin{equation}
\begin{split}
    \E \inner{\hat{Z}^{(i)}}{\hat{Z}^{(j)}} &= \E \hat{Z}^{(i)} \cdot \hat{Z}^{(j)} \\
    &= \E \tfrac{Z^{(i)}}{\|Z^{(i)}\|} \cdot \tfrac{Z^{(j)}}{\|Z^{(j}\|} \\
    &= \E \tfrac{1}{\|Z^{(i)}\| \|Z^{(j)}\|} \inner{Z^{(i)}}{Z^{(j)}} \\
    &= \tfrac{\sqrt{D}}{\sqrt{D} \sqrt{D}} = \tfrac{1}{\sqrt{D}}.
\end{split}
\end{equation}
Taking the limit shows that any two random vectors are orthogonal in high-dimensional isotropic vector spaces
\begin{equation}
    \lim_{D \to \infty} \E \inner{\hat{Z}^{(i)}}{\hat{Z}^{(j)}} = 0.
\end{equation}
As orthogonality is invariant to vector length
\begin{equation}
    \E \inner{\hat{Z}^{(i)}}{\hat{Z}^{(j)}} = \E \inner{Z^{(i)}}{Z^{(j)}} = \tfrac{1}{\sqrt{D}}.
\end{equation}
Noting that inner product $\inner{Z^{(i)}}{Z^{(j)}}$ equals cosine distance similarity $sim(Z^{(i)}, Z^{(j)})$ for Euclidean spaces completes the proof.
\end{proof}


\subsection{Proof for Lemma~\ref{lemma:hyperspherical_cap}}
\label{app:lemma_2}

\begin{proof}
Supposing the optimal compositional semantic embedding $z^*$ is found given a set of $K$ sub-semantic embeddings $\mathcal{Z} = \{ z^{(1)}, \ldots, z^{(K)} \}$ such that
\begin{equation}
    z^* = \argmax \tfrac{1}{K} \sum_{i=1}^K sim(z^*, z^{(k)}) - \E sim(z^*, z')
\end{equation}
where $z'$ is a semantic embedding of any unrelated object description.

Note that the sub-semantics $\mathcal{Z}$ can be ordered by similarity with $z^*$, and that the least similar sub-semantic $z_{min}$ and its similarity value $\epsilon$ is known
\begin{equation}
    z_{min} = \argmin sim(z^*, z) \forall z \in \mathcal{Z}.
\end{equation}
\begin{equation}
    \epsilon = sim(z^*, z_{min}).
\end{equation}

A hyperspherical cap $S^{D-1}_{cap}$ is defined by $z^*$ as the normal center vector and the angle $\theta_{min}$ between $z^*$ and $z_{min}$
\begin{equation}
    S^{D-1}_{cap} = \{ z \in \mathbb{R}^D : \|z\| = 1, \theta_z \le \theta_{min} \}
\end{equation}
\noindent
where the angles $\theta$ are related to similarities by
\begin{equation}
    \theta_z = \arccos(sim(z^*, z))
\end{equation}
\begin{equation}
\label{eq:theta_min}
    \theta_{min} = \arccos(sim(z^*, z_{min})).
\end{equation}

Since
\begin{equation}
\label{eq:all_z_more_sim_than_z_min}
    sim(z^*, z) \ge sim(z^*, z_{min}) \Leftrightarrow \theta_z \le \theta_{min} \;\; \forall z \in \mathcal{Z}
\end{equation}
\begin{equation}
    sim(z^*, z^*) = 1 \Leftrightarrow \theta_{z^*} = 0 < \theta_{min}
\end{equation}
\noindent
all $z \in \mathcal{Z}$ and $z^*$ are in $S^{D-1}_{cap}$.

\end{proof}


\subsection{Proof for Theorem~\ref{theorem:discoverability}}
\label{app:theorem_1}

\begin{proof}
The optimal compositional semantic embedding $z^* \in \mathbb{R}^D$ representing a set of $K$ sub-semantics $z \in \mathcal{Z}$ in a uniform distribution over the unit hypersphere $\mathit{U}(S^{D-1})$ is
\begin{equation}
\label{eq:maximize_sim}
    z^* = \argmax \sum_{i=1}^K sim(z^*, z^{(i)}) = \argmax \sum_{i=1}^K (z^*)^T z^{(i)}.
\end{equation}

Maximizing cosine distance similarity $sim(z^*, z)$ is equivalent to minimizing squared distance $||z^* - z||^2$ on the unit hypersphere as
\begin{equation}
\begin{split}
    \min \sum_{i=1}^K || z^* - z^{(i)}&||^2 = \sum_{i=1}^K (z^* - z^{(i)})^T (z^* - z^{(i)}) \\
    &= \min \sum_{i=1}^K \left[ ||z^*|| - 2 (z^*)^T z^{(i)} + ||z^{(i)}|| \right] \\
    &= \min \sum_{i=1}^K \left[ 2 - 2 (z^*)^T z^{(i)}  \right] \\
    &= \min \left[ 2K - 2\sum_{i=1}^K (z^*)^T z^{(i)} \right] \\
    &\propto \min \left[- \sum_{i=1}^K (z^*)^T z^{(i)} \right] \\
    &= \max \sum_{i=1}^K (z^*)^T z^{(i)}
\end{split}
\end{equation}

The vector $z^*$ maximizing \eqref{eq:maximize_sim} can thus be found from the derivative with respect to the vector $z^*$
\begin{equation}
\label{eq:square_dist_derivative}
    \frac{d}{d z^*} \sum_{i=1}^K || z^* - z^{(i)}||^2 = 0.
\end{equation}

To apply the general chain rule~\cite{deisenroth2020MML}, we rewrite \eqref{eq:square_dist_derivative} with variable substitution so that each operation in the function is factored into single variable components for easily finding partial differentials:
%
%
\begin{equation}
\label{ex:var_subtstitution}
\begin{split}
    \sum_{i=1}^K || z^* - z^{(i)}||^2 = \: &g =  \sum_{i=1}^K ||f||^2 \\
    &f = z^* - z^{(i)}.
\end{split}
\end{equation}
Applying the chain rule and noting that $||f||^2 = f^T f$ gives
%
\begin{equation}
\label{eq:square_dist_chain_rule}
\begin{split}
    \frac{\partial g}{\partial z^*} &= \frac{\partial g}{\partial f} \frac{\partial f}{ \partial z^*} \\
    &= \sum_{i=1}^K 2 f^T \: \frac{\partial}{\partial z^*} \left( z^* - z^{(i)} \right) \\
    &= 2 \sum_{i=1}^K (z^* - z^{(i)})^T \left[ \tfrac{\partial}{\partial z_1^*} (z^* - z^{(i)}), \ldots, \tfrac{\partial}{\partial z_D^*} (z^* - z^{(i)})^T \right] \\
    &= 2 \sum_{i=1}^K (z^* - z^{(i)})^T \left[ e_1, \ldots, e_D \right] \\
    &= 2 \left[ \sum_{i=1}^K (z_1^* - z_1^{(i)}), \ldots, \sum_{i=1}^K (z_D^* - z_D^{(i)}) \right]^T = 0
\end{split}
\end{equation}

\noindent where $e_d$ is the one-hot vector with the $d^{\text{th}}$ element set to 1. Equation \eqref{eq:square_dist_chain_rule} is an element-wise system of equations stating that for every $d^{\text{th}}$ element
\begin{equation}
    \sum_{i=1}^K (z_d^* - z_d^{(i)}) = 0
\end{equation}
\noindent meaning the optimal $z^*$ maximizing \eqref{eq:maximize_sim} equals the centroid of the sub-semantics $z^{(i)} \in \mathcal{Z}$
\begin{equation}
\label{eq:centroid}
    z^* = \frac{1}{K} \sum_{i=1}^K z^{(i)}.
\end{equation}

To prove $z^*$ specified by \eqref{eq:centroid} satisfies Definition~\ref{def:compositionality} we write
\begin{equation}
\label{eq:z_star_sim_w_centroid}
    \E sim(z^*, z) = \E \left[ \left(\tfrac{1}{K} \sum_{i=1}^K z^{(i)} \right) \cdot z \right] = \tfrac{1}{K} \sum_{i=1}^K \E z^{(i)} \cdot z.
\end{equation}

As $z$ equals one of the $z^{(i)} \in \mathcal{Z}$ we can assume $z = z^{(k)}$ without loss of generality and expand the sum in \eqref{eq:z_star_sim_w_centroid} as 
\begin{multline}
\label{eq:z_star_sim_expanded}
    \E sim(z^*, z) = \tfrac{1}{K} \left( \E [z^{(1)} \cdot z^{(k)}] + \ldots \right.\\ 
    \left.+ \E [z^{(k)} \cdot z^{(k)}] + \ldots + \E [z^{(K)} \cdot z^{(k)}] \right)
\end{multline}

We find a lower bound for \eqref{eq:z_star_sim_expanded} by applying Lemma~\ref{lemma:expected_sim} and noting that the expected similarities $sim(z^{(i)}, z^{(j)}) \forall z^{(i)}, z^{(j)} \in \mathcal{Z}$ must be higher or equal to random vectors, and that $z^{(k)} \cdot z^{(k)} = 1$
\begin{equation}
\label{eq:z_star_sim_bound}
\begin{split}
    \E sim(z^*, z) &\ge \tfrac{1}{K} \left( D^{-\frac{1}{2}} + \ldots + 1 + \ldots + D^{-\frac{1}{2}} \right) \\
    &= \tfrac{1}{K} \left( (K-1) D^{-\frac{1}{2}} + 1 \right).
\end{split}
\end{equation}

Substituting the bound \eqref{eq:z_star_sim_bound} into Definition~\ref{def:compositionality} and applying Lemma~\ref{lemma:expected_sim} on the RHS 
\begin{equation}
\label{eq:z_star_inequality}
    \E sim(z^*, z) \ge \tfrac{1}{K} \left( (K-1) D^{-\frac{1}{2}} + 1 \right) > D^{-\frac{1}{2}}.
\end{equation}

Rearranging the two leftmost inequalities in \eqref{eq:z_star_inequality}
\begin{gather}
    (K-1) D^{-\frac{1}{2}} + 1 > K D^{-\frac{1}{2}} \\
    K D^{-\frac{1}{2}} - D^{-\frac{1}{2}} + 1 - K D^{-\frac{1}{2}} > 0 \\
    - D^{-\frac{1}{2}} > - 1 \\
    D^{-\frac{1}{2}} < 1 \\
    \sqrt{D} > 1
\end{gather}
which is true for $D > 1 $ and thus proves Theorem~\ref{theorem:discoverability}.

\end{proof}


\subsection{Proof for Theorem~\ref{theorem:prob_of_compositionality}}
\label{app:theorem_2}

\begin{proof}

A random vector $z'$ sampled from the uniform distribution over the unit hypersphere $\mathit{U}(S^{D-1})$ is equally likely to be a point anywhere on $S^{D-1}$. The probability $z'$ is sampled in a particular surface region $A_{D,r}$ is
\begin{equation}
\label{eq:prob_in_hyper_cap}
    P(z' \in A_{D,r}) = \tfrac{A_{D,r}}{A_D}
\end{equation}
\noindent where $A_D$ is the total surface region.

The probability $z'$ is sampled into the surface region defined by the hyperspherical cap $S^{D-1}_{cap}$ with surface area $A_{cap}$ given in Lemma~\ref{lemma:hyperspherical_cap} is therefore
\begin{equation}
    P(z' \in S^{D-1}_{cap}) = \tfrac{A_{cap} }{A_D}.
\end{equation}

The surface area ratio of a hyperspherical cap~\cite{Li2011HypersphereCap} is
\begin{equation}
\label{eq:hyper_surface_area_ratio}
    A_{D,r} = \tfrac{1}{2} A_D I_{\sin^2(\theta)}(\tfrac{D-1}{2}, \tfrac{1}{2})
\end{equation}
\noindent where $I_x(a, b)$ is the regularized incomplete beta function.

Substituting \eqref{eq:hyper_surface_area_ratio} into \eqref{eq:prob_in_hyper_cap} gives
\begin{equation}
    P(z' \in S^{D-1}_{cap}) = \tfrac{1}{2} I_{\sin^2(\theta)}(\tfrac{D-1}{2}, \tfrac{1}{2}).
\end{equation}

The probability that $z'$ is not sampled in $S^{D-1}_{cap}$ is
\begin{equation}
\label{eq:prob_z'_not_in_cap}
\begin{split}
    P(z' \notin S^{D-1}_{cap}) &= 1 - P(z' \in S^{D-1}_{cap}) \\
    &= 1 - \tfrac{1}{2} I_{\sin^2(\theta)}(\tfrac{D-1}{2}, \tfrac{1}{2}).
\end{split}
\end{equation}

By Lemma~\ref{lemma:hyperspherical_cap} and \eqref{eq:theta_min} we know 
\begin{equation}
    \forall z' \;\;  sim(z^*, z_{min}) \ge sim(z^*, z') \Leftrightarrow z' \notin S^{D-1}_{cap}.
\end{equation}

Substituting the bound $sim(z^*,z_{min})$ by \eqref{eq:all_z_more_sim_than_z_min} gives
\begin{equation}
\label{eq:subst_bound}
    \forall z', z \in \mathcal{Z} \;\;  sim(z^*, z) \ge sim(z^*, z') \Leftrightarrow z' \notin S^{D-1}_{cap}.
\end{equation}

Substituting the LHS of \eqref{eq:subst_bound} into \eqref{eq:prob_z'_not_in_cap} and recollecting \eqref{eq:theta_min} proves Theorem~\ref{theorem:prob_of_compositionality}.

\end{proof}


\subsection{Proof for Proposition~\ref{proposition:discoverability_2}}
\label{app:proposition_1}

\begin{proof}
Non-uniformity means the distribution of vectors is not maximally dispersed over the hypersphere~\cite{wang2020hypersphere}. Recalling Lemma~\ref{lemma:expected_sim} for uniform distributions, the expected similarity of two non-uniformly distributed independent vectors $Z^{(i)}, Z^{(j)} \sim p(Z)$ is therefore
\begin{equation}
\label{eq:non-uniform_random_sim}
    \E sim(Z^{(i)}, Z^{(j)}) = C \ge \tfrac{1}{\sqrt{D}}.
\end{equation}

By substituting \eqref{eq:non-uniform_random_sim} in Definition~\ref{def:compositionality} gives
\begin{equation}
\label{eq:non-uniform_comp_sim}
    \E sim(z^*, z) > \E sim(z^*, z') = C.
\end{equation}

Expanding the LHS of \eqref{eq:non-uniform_comp_sim} using the same idea as in \eqref{eq:z_star_sim_w_centroid} and \eqref{eq:z_star_sim_expanded}
\begin{equation}
\label{eq:non-uniform_comp_sim_bound}
    \E sim(z^*, z) \ge \tfrac{1}{K}\left[ (K-1) C + 1 \right].
\end{equation}

Substituting \eqref{eq:non-uniform_comp_sim_bound} into \eqref{eq:non-uniform_comp_sim}
\begin{gather}
    \tfrac{1}{K}\left[ (K-1) C + 1 \right] > C \\
    KC - C +1 > KC \\
    -C > -1 \\
    C < 1.
\end{gather}

Since $sim(z^{(i)}, z^{(j)}) \in [-1, 1[ \;\; s.t. z^{(i)} \neq z^{(j)}$ the inequality \eqref{eq:non-uniform_comp_sim} is true  for all distributions $p(z)$ except the singular distribution and thus proves Proposition~\ref{proposition:discoverability_2}.

\end{proof}


\subsection{Proof for Proposition~\ref{proposition:gradient_descent}}
\label{app:proposition_2}

\begin{proof}

The global convergence guarantee for convex optimization problems~\cite{boyd2004convex_optimization} proves that for any convex function $f$ is guaranteed that the value $z^{*(t)}$ converges to the optimal value $z^*$
\begin{equation}
    \lim_{t \rightarrow \infty} f(z^{*(t)}) = f(z^*)
\end{equation}
given a sufficiently small learning rate $\lambda$.

We prove that the cosine similarity optimization objective~\eqref{eq:maximize_sim} is a convex problem by noting that the set~\eqref{eq:set_of_zs} is convex and show that the Hessian matrix
\begin{equation}
    H\left( f(z^*) \right) = \nabla^2_{z^*} f(z^*) = \begin{bmatrix} \frac{\partial}{\partial z^*_i \partial z^*_j} f(z^*) \end{bmatrix}
\end{equation}
is a positive semidefinite matrix~\cite{deisenroth2020MML}. Note that $f(z^*)$ substitutes $\sum_{k=1}^K sim(z^*, z^{(k)})$. Recalling the form of the first partial derivatives~\eqref{eq:square_dist_chain_rule} and taking another partial derivative for an arbitrary element
\begin{equation}
    \tfrac{\partial}{\partial z^*_i \partial z^*_j} f(z^*) = \tfrac{\partial}{\partial z^*_i} \left[ K z^*_j - \sum_{k=1}^K z^{(k)}_j\right] = K \mathbf{1}_{i=j}.
\end{equation}
The Hessian matrix is thus the scaled identity matrix
\begin{equation}
    H\left( f(z^*) \right) = \begin{bmatrix} K e_{i} \end{bmatrix} = K I_D
\end{equation}
meaning $f(z^*)$ is a convex function.

Finally noting that
\begin{equation}
    \E \left[ \frac{1}{L}\sum_{i=1}^L z \in \mathcal{Z}^{(t)} \right] = \frac{1}{K} \sum_{i=1}^K z^{(i)} \in \mathcal{Z}, \; \mathcal{Z}^{(t)} \subseteq \mathcal{Z}
\end{equation}
shows that the optimal convergence value obtained by optimizing~\eqref{eq:maximize_sim} by gradient descent results in the optimal compositional semantic embedding $z^*$ obtained by \eqref{eq:centroid}.

\end{proof}

}


 
%

\bibliographystyle{IEEEtran}
\bibliography{refs}

\vspace{-10mm}
\begin{IEEEbiographynophoto}{Robin Karlsson}
(Student Member, IEEE) received a BSc. degree from the School of Engineering, Aalto University, Finland, and an MSc. degree from the Graduate School of Frontier Science, University of Tokyo, Japan. From 2018 to 2021 he worked as an autonomous vehicle research scientist at Ascent Robotics and TIER IV. He is currently pursuing a Ph.D. degree at the Graduate School of Informatics, Nagoya University, Japan. His research interest includes neurosymbolic AI, world representations for general-purpose mobile robotics, machine reasoning, and AGI. His contributions include two international conference best paper awards, a national student competition 1st place, and the IEEE ITSS Young Researcher Award.
\end{IEEEbiographynophoto}
\vspace{-10mm}
\begin{IEEEbiographynophoto}{Francisco Lepe-Salazar}
pursued his graduate studies at Waseda University, Japan. He is currently the Managing Director of Ludolab, the Director of the Observatorio Nacional de la Industria de los Videojuegos with DevsVJ MX, a Creator and an Organizer with International Contest Games4Empowerment, a Founding Member of the Educational Program Código Frida, and a Teacher with the University of Colima, Mexico, and the Deggendorf Institute of Technology, Germany. His research interests include cognitive science, human–computer interaction, and user empowerment. His contributions include various publications at top international venues.
\end{IEEEbiographynophoto}
\vspace{-10mm}
\begin{IEEEbiographynophoto}{Kazuya Takeda}
(Governors member, IEEE ITS Society; Governors member, APSIPA; Fellow, IEICE) serves as a Vice President of Nagoya University and Professor at Nagoya University's Institute of Innovation for Future Society and Graduate School of Informatics. He is also a Director at Tier IV, Inc. Dr. Takeda earned his Bachelor's, Master's, and Ph.D. from Nagoya University in 1983, 1985, and 1993, respectively. He has held positions at ATR (Advanced Telecommunication Research Laboratories) and KDD R\&D Lab, in addition to being a visiting scientist at MIT.
His research primarily focuses on signal processing and machine learning of behavior signals and their applications. With over 150 journal papers, 9 co-authored/co-edited books, and 15 patents to his name, Dr. Takeda is a prolific contributor to his field. His achievements include the 2020 IEEE ITS Society Outstanding Research Award and six best paper awards from IEEE international conferences and workshops, in addition to various domestic awards.
\end{IEEEbiographynophoto}

\end{document}